\documentclass[journal]{IEEEtran}
\usepackage{amsmath,amsfonts}
\usepackage{array}
\usepackage[caption=false,font=normalsize,labelfont=sf,textfont=sf]{subfig}
\usepackage{textcomp}
\usepackage{stfloats}
\usepackage{url}
\usepackage{verbatim}
\usepackage{graphicx}
\usepackage{cite}
\usepackage[utf8]{inputenc}
\usepackage{wrapfig, setspace, booktabs}

\usepackage{microtype}
\usepackage{amssymb}
\usepackage{pifont}
\usepackage{xcolor}
\usepackage{enumitem}
\usepackage{algorithm,algpseudocode}
\usepackage{amsthm}
\usepackage[symbol]{footmisc}

\usepackage{mathtools}

\newtheorem{proposition}{Proposition}

\theoremstyle{definition}

\newtheorem*{assumption*}{\assumptionnumber}

\providecommand{\assumptionnumber}{}


\DeclarePairedDelimiter{\abs}{\lvert}{\rvert}%
\DeclarePairedDelimiterX{\inp}[2]{\langle}{\rangle}{#1, #2}

\makeatletter
\let\oldabs\abs
\def\abs{\@ifstar{\oldabs}{\oldabs*}}
\makeatother

\usepackage[hidelinks]{hyperref}
\hypersetup{
  pdftitle={LoomSC: Scalable Deep Subspace Clustering with Projector Factorization and Exact Spectral Reduction},
  pdfauthor={Nairouz Mrabah, Youssef Melki, Mohamed Bouguessa, Riadh Ksantini, Shakeeb Murtaza, Tehseen Zia},
  pdfsubject={Scalable deep subspace clustering},
  pdfkeywords={subspace clustering, deep clustering, projector factorization, spectral clustering}
}

\ifdefined\pdfminorversion\pdfminorversion=7\fi
\begin{document}
\raggedbottom

\title{LoomSC: Scalable Deep Subspace Clustering with Projector Factorization and Exact Spectral Reduction}

\author{Nairouz Mrabah, Youssef Melki, Mohamed Bouguessa, Riadh Ksantini, Shakeeb Murtaza, and Tehseen Zia%
\thanks{N. Mrabah and S. Murtaza are with the \'{E}cole de technologie sup\'{e}rieure (\'{E}TS), Montr\'{e}al, QC, Canada. Y. Melki and R. Ksantini are with the University of Bahrain (UOB), Kingdom of Bahrain. M. Bouguessa is with the Universit\'{e} du Qu\'{e}bec \`{a} Montr\'{e}al (UQAM), Montr\'{e}al, QC, Canada. T. Zia is with COMSATS University Islamabad, Islamabad, Pakistan.}}

\markboth{LoomSC: Scalable Deep Subspace Clustering}{LoomSC: Scalable Deep Subspace Clustering}

\maketitle

\begin{abstract}
Dense self-expression matrices and full-affinity spectral clustering limit the scalability of subspace clustering. We introduce the Latent Orthogonal Optimization Model for Subspace Clustering (LoomSC), a framework that addresses both bottlenecks through projector factorization and exact spectral reduction. Motivated by the spectral structure of least-squares regression, LoomSC jointly learns latent features and a projector self-representation through two thin factors. Alternating Procrustes and least-squares updates preserve the sample factor's orthogonality while keeping the coefficient matrix implicit. We construct a nonnegative quadratic affinity that preserves the projector's support. An exact feature map then reduces its normalized spectral problem to an eigenproblem whose dimension depends only on the factor width. Neither the full affinity nor the sample Laplacian needs to be formed. Our analysis quantifies the projector approximation and identifies conditions for subspace preservation and within-subspace connectivity. For fixed dimensions and iteration budgets, the complete pipeline has linear time and memory complexity in the number of samples. Across five image-clustering benchmarks, LoomSC ranks first or second in all 15 dataset--metric comparisons against 9 state-of-the-art baselines. Its mean accuracy exceeds the highest baseline mean by 6.66 percentage points. Synthetic experiments scale to 500,000 samples while maintaining at least 99.8\% accuracy.
\end{abstract}

\section{Introduction}

In several cases, high-dimensional data can be modeled by a generation process from the union of multiple low-dimensional linear subspaces. In this context, subspace clustering is a prominent unsupervised strategy that aims to partition the data into clusters by identifying these subspaces, where each cluster is defined by exactly one low-dimensional subspace. The subspace assumption is widely used in visual analysis and has been successfully exploited in several applications, such as image segmentation \cite{paper2}, motion segmentation \cite{paper3, paper5}, and face clustering \cite{paper5, paper4}. For example, face images of a person can be represented by a low-dimensional subspace for a fixed pose and under Lambertian reflectance \cite{paper1}.

The most recent subspace clustering methods \cite{paper8, paper6, paper7} follow a two-step strategy. First, these methods learn an affinity matrix that captures the similarity between points within the same subspace. Second, a graph clustering approach, such as spectral clustering, is used to identify the clusters. The affinity matrix is constructed according to the self-expression property \cite{paper7, paper9} of the subspaces. More precisely, each data point can be expressed as a linear combination of points in the same subspace. Given a dataset $X \in \mathbb{R}^{d_{x} \times n}$ with $n$ data points of dimensionality $d_{x}$, the optimization problem associated with this strategy is:

\begin{equation}
  \begin{aligned}
        \min_{_{C}} \frac{1}{2} \left \| X - XC \right \|_{F}^{2} + \lambda \, \theta(C),
  \end{aligned}
\label{eq:1}
\end{equation}

\noindent where $C = (c_{ij}) \in \mathbb{R}^{n \times n}$ is a matrix that captures the self-expression property based on the first term; $\left \| . \right \|_{F}$ denotes the Frobenius norm, which is defined for any matrix $M= (m_{ij})$ by $\left \| M \right \|_{F} = ( \sum_{i=1} \sum_{j=1} m_{ij}^{2})^{\frac{1}{2}}$; the second term, controlled by the function $\theta$, serves as a regularization to avoid trivial solutions (e.g., $C = I$) and to adjust the structure of the self-representation matrix $C$; $\lambda \in \mathbb{R}_{+}^{*}$ is a balancing hyperparameter between the two terms. Based on the self-representation matrix, the affinity graph is constructed by computing its adjacency matrix $A_{\mathrm{abs}} = (a_{ij}^{\mathrm{abs}}) \in \mathbb{R}^{n \times n}$, such that $A_{\mathrm{abs}} = \frac{\left | C \right | + \left | C^{T} \right |}{2}$.

Different regularizers encode complementary structural preferences. Sparse Subspace Clustering (SSC)~\cite{paper7} promotes sparse coefficients through an $\ell_1$ penalty. Low-Rank Representation (LRR)~\cite{paper11} favors a low-rank coefficient matrix through nuclear-norm regularization. Multi-Subspace Representation (MSR)~\cite{paper10} combines these penalties, whereas Block Diagonal Representation (BDR)~\cite{paper8} directly encourages a block structure. Least Squares Regression (LSR)~\cite{paper9} uses squared Frobenius regularization and admits a closed-form solution. Its grouping effect encourages correlated samples to receive similar coefficients. These formulations seek to suppress cross-subspace connections while retaining informative relations within each subspace. Both requirements matter because a sparse affinity can separate samples from different subspaces yet fragment samples from the same subspace~\cite{paper12}.

The self-expression framework nevertheless introduces two distinct computational bottlenecks. Explicitly storing a dense coefficient matrix requires $\mathcal O(n^2)$ memory, even when its estimation admits an efficient solver. Constructing a full affinity retains this quadratic storage cost. The spectral stage adds a second bottleneck. A dense eigendecomposition requires $\mathcal O(n^3)$ operations, while partial eigensolvers still perform repeated products with the sample-level affinity. Consequently, accelerating coefficient estimation alone does not ensure a scalable clustering pipeline. Both self-expression and spectral assignment must avoid dense sample-by-sample computations.

Several strategies reduce these costs. Sampling methods replace the full dictionary with a smaller set of representatives~\cite{paper13,paper14,paper15,paper36}. Subspace-preserving selection seeks representatives that span the underlying subspaces~\cite{paper16}. Sparse solvers accelerate coefficient estimation through greedy pursuit~\cite{paper17}, active-set selection~\cite{paper19}, or stochastic dictionaries~\cite{paper18}. Structure-aware methods exploit compact graph models~\cite{paper20} or structured affinity estimation~\cite{paper22}. Neural coefficient predictors~\cite{paper21} provide another route by learning a mapping that replaces repeated regression.

Representative-based reduction is particularly attractive because it replaces an all-sample dictionary with $m\ll n$ representatives. Each sample is then described by $m$ coefficients instead of $n$. Anchor-graph methods can also obtain the spectral embedding from a compact factor~\cite{paper15,paper36}. Their effectiveness depends on how well the representatives capture the subspaces in the supplied feature space. Learning the coefficients improves the graph, but does not by itself adapt the image representation. This motivates a compact model in which both the representation and the directions used for self-expression are learned from the data.

The representation itself is important because complex appearance variation and data corruption can weaken the linear subspace structure of raw observations. Kernel methods address nonlinear structure through an implicit feature mapping~\cite{paper27}. Deep subspace clustering~\cite{paper28,paper32,paper34} instead couples an auto-encoder with self-expression in the latent space. The reconstruction objective preserves information about the inputs. The self-expression objective encourages latent features that conform to a subspace model. Joint optimization therefore adapts the features to clustering rather than imposing self-expression on a fixed representation.

Joint learning also makes computational compactness essential. A dense self-expression layer introduces $n^2$ sample-dependent coefficients that must be optimized together with the encoder~\cite{paper28,paper32}. Neural parameterizations and mini-batch training can reduce training costs~\cite{loom_pro_dsc}, but do not by themselves eliminate full-affinity spectral inference. Our goal is to retain both pairwise self-expression and normalized spectral clustering while learning nonlinear representations, without constructing dense sample-level matrices at either stage.

To this end, we introduce the Latent Orthogonal Optimization Model for Subspace Clustering (LoomSC). Motivated by the spectral structure of LSR, LoomSC represents self-expression through a projector $C=PP^\top$. The sample factor $P\in\mathbb R^{n\times m}$ has orthonormal columns and a prescribed width $m$. An equivalent two-factor objective separates latent directions from sample coefficients and couples both to the auto-encoder. Alternating Procrustes and least-squares updates refine these factors while preserving orthogonality. Neither the coefficient matrix nor a sample Gram matrix needs to be formed. The latent directions are updated using all samples rather than remaining fixed to selected representatives. Our analysis quantifies the projector approximation to LSR and identifies conditions under which the full row-space projector preserves subspace memberships and within-subspace connectivity.

LoomSC extends this compact representation to spectral assignment through a nonnegative quadratic affinity. Squaring the projector coefficients preserves their zero pattern and thus retains the connections encoded by self-expression. An exact feature map represents this affinity using $q=m(m+1)/2$ features per sample. Its normalized spectral problem then reduces to a $q$-dimensional eigenproblem. This reduction avoids forming either the full affinity or the sample Laplacian. For fixed dimensions and iteration budgets, the complete pipeline has linear time and memory complexity in $n$. The same compact representation therefore supports feature learning, self-expression, and global spectral assignment.

\textbf{Contributions.} \noindent\textbf{(1)} We introduce a two-factor projector model for joint feature and self-expression learning with orthogonality-preserving updates.
\noindent\textbf{(2)} We construct a support-preserving nonnegative affinity and derive an exact reduction of its normalized spectral problem without forming full pairwise matrices.
\noindent\textbf{(3)} We quantify the projector approximation to LSR and establish conditions for subspace preservation and connectivity. We also prove linear time and memory bounds for the complete pipeline under fixed dimensions and iteration budgets.
\noindent\textbf{(4)} Across five image-clustering benchmarks, LoomSC ranks first or second in all $15$ dataset--metric comparisons against nine baselines. Its mean accuracy exceeds the highest baseline mean by $6.66$ percentage points. Experiments on synthetic data reach $500{,}000$ samples while maintaining at least $99.8\%$ accuracy throughout the tested range.

\section{Related work}
\label{sec:related-work}
In line with the focus of this work, we discuss the advantages and limitations of two existing subspace clustering categories: the scalable subspace clustering category and the deep subspace clustering category. In that respect, we highlight the pertinence of introducing a subspace clustering approach that combines the advantages of both categories.

\subsection{Scalable subspace clustering:}

Least-squares regression (LSR)~\cite{paper9} provides an efficient alternative to sparse coefficient estimation. Its squared Frobenius regularization encourages a grouping effect among correlated samples and admits a closed-form solution. However, solving the regression analytically does not eliminate the dense $n\times n$ coefficient matrix, where $n$ is the number of samples. Its explicit construction and storage remain quadratic in $n$. Thus, improving the regression solver alone does not resolve the full-affinity bottleneck. LoomSC instead adopts a prescribed-width projector model and learns its thin factors directly. This changes the representation of self-expression rather than merely accelerating a dense coefficient solve.

Sparse methods reduce storage by retaining a small number of coefficients per sample. SSC-OMP~\cite{paper17} replaces $\ell_1$ optimization with orthogonal matching pursuit and establishes subspace-preserving guarantees under suitable conditions. EnSC~\cite{paper19} combines $\ell_1$ and squared $\ell_2$ regularization to balance subspace preservation and graph connectivity. Its oracle-guided active-set solver restricts optimization to candidate dictionary elements. Both approaches substantially accelerate coefficient estimation. Nevertheless, their dictionary searches still compare target residuals with the full dataset. Repeating these searches across samples retains a quadratic sample-size dependence for fixed sparsity and search budgets. Sparse spectral solvers reduce the subsequent graph-processing cost, but do not remove this coefficient-estimation cost. LoomSC avoids these separate full-dictionary searches by optimizing a shared factorization through thin matrix products and an orthogonal Procrustes update.

Stochastic sparse subspace clustering (SSSC)~\cite{paper18} addresses weak within-subspace connectivity through dictionary dropout. For normalized samples, the expected dropout objective induces an additional squared $\ell_2$ penalty, which encourages denser representations. Its scalable implementation combines coefficients from multiple reduced dictionaries through a consensus formulation. This improves connectivity while preserving a sparse affinity. However, with a fixed retained fraction, each dictionary still grows with $n$, and the reduced regression must be repeated for every target. Parallel execution accelerates these subproblems without changing their aggregate arithmetic complexity. LoomSC obtains compactness through a common learned latent factor rather than repeated stochastic dictionary reduction. The resulting sample relations are represented jointly without an ensemble of coefficient estimates.

Representative-based approaches reduce the dictionary size more directly. LMVSC~\cite{paper15} learns nonnegative sample-to-anchor coefficients and computes its spectral embedding from a compact normalized factor. SGL~\cite{paper36} further couples anchor-graph learning with a spectral constraint that targets $k$ connected components. These methods achieve linear sample-size scaling for fixed anchor and optimization budgets. Their graph weights are learned, but their reconstruction is performed using supplied features and an anchor dictionary. Consequently, they do not jointly adapt the image representation to the self-expression objective. LoomSC preserves a compact representation while learning both the encoder and the latent factor. After initialization, the columns of the latent factor are updated from all samples rather than remaining fixed anchor representatives. Unlike the nonnegative anchor weights in LMVSC and SGL, its sample-factor coefficients can be signed. This permits linear combinations beyond convex anchor assignments.

Selective sampling-based scalable sparse subspace clustering (S$^5$C)~\cite{loom_s5c} selects representatives using stochastic estimates of subgradient violations. It then solves restricted sparse regressions for all samples. Its spectral stage exploits a graph with only $O(nm)$ nonzero entries, where $m$ is the representative budget. Thus, S$^5$C reduces both regression and spectral costs without requiring a dense affinity. However, its representation remains tied to selected samples in the feature space. LoomSC replaces this exemplar dictionary with adaptive latent directions and jointly refines the features used to construct them.

A-DSSC~\cite{paper22} addresses a different source of difficulty: the compatibility between self-expression coefficients and spectral normalization. It first estimates coefficients and then obtains a doubly stochastic affinity through quadratically regularized optimal transport. A support-restricted solver exploits sparsity to accelerate this second stage. This gives a principled alternative to heuristic coefficient normalization, but retains separate coefficient-estimation and transport problems. LoomSC instead derives the affinity and its normalization from the same learned projector. For its squared affinity, each degree equals the squared norm of the corresponding sample-factor row. Thus, degree computation requires only the thin factor and no additional transport optimization.

Compact spectral computation also appears in scalable attributed-graph subspace clustering (SAGSC)~\cite{loom_sagsc}. SAGSC constructs a projector from graph-smoothed features and uses a quadratic kernel to obtain a nonnegative factorized affinity. This permits spectral clustering without explicitly assembling pairwise weights. LoomSC applies this spectral principle to jointly learned image representations rather than fixed graph-smoothed features. Its two-factor objective allows the projector to evolve with the encoder, and its analysis connects the learned factors to dominant latent directions and subspace structure. The distinction is therefore the integration of adaptive representation learning with a compact spectral construction, rather than factorization alone.

\subsection{Deep subspace clustering:}

DSCNet~\cite{paper28} integrates a differentiable self-expression layer into a convolutional auto-encoder. This couples feature learning with the reconstruction of each latent sample from the others. S$^2$ConvSCN~\cite{paper32} additionally uses spectral assignments to supervise feature learning and coefficient estimation. These formulations make the representation responsive to clustering, but retain a sample-to-sample self-expression matrix. Additional supervision does not remove its quadratic parameter storage. LoomSC retains joint feature and self-expression learning while replacing this layer with two thin factors. Its factor updates admit closed-form solutions, and their dimensions are controlled by the prescribed factor width.

Neural coefficient prediction offers another route to scalability. SENet~\cite{paper21} parameterizes pairwise coefficients with a neural network that can generalize to unseen samples. Its objective learns a coefficient-prediction rule for the supplied representation rather than jointly learning a new feature space for self-expression. This separates the number of network parameters from the number of coefficient entries. However, a compact predictor does not itself provide a compact spectral representation of the resulting affinity. LoomSC instead learns a sample factor whose algebraic structure supports both self-expression and the subsequent spectral reduction.

PRO-DSC~\cite{loom_pro_dsc} jointly learns latent representations and self-expression coefficients. A log-determinant regularizer controls representation collapse under the conditions established in its analysis. Its implementation predicts coefficients through neural mappings and Sinkhorn normalization, which supports mini-batch training. The distinction from LoomSC arises at full-set clustering. PRO-DSC constructs the coefficient matrix over the clustering set and applies spectral clustering to the resulting affinity. For $n$ samples, this inference route retains quadratic matrix storage despite efficient training. LoomSC does not reconstruct a full coefficient matrix after learning. The final sample factor directly defines a nonnegative affinity whose normalized spectral problem is solved in a feature space determined by the factor width. Thus, the computational reduction extends through spectral assignment rather than ending with network training.

Other deep methods reduce cost by changing how the final partition is obtained. The large-scale extension of pseudo-supervised deep subspace clustering~\cite{paper34} trains on a subset and assigns the remaining latent samples through nearest-neighbor classification. LoomSC instead retains an explicit model of sample relations while avoiding their dense storage. All samples participate in factor learning, and the final partition is obtained by spectral clustering of the learned affinity rather than by extending labels from a training subset. Its efficiency follows from representing that affinity compactly while preserving the self-expression formulation.

\section{LoomSC}
\label{sec:loomsc}

The proposed method develops a compact self-representation by exploiting the spectral structure of least-squares regression (LSR). We first analyze the LSR solution and motivate a projector approximation. We then express this approximation through two thin factors that are learned jointly with a deep auto-encoder. Finally, we construct a nonnegative affinity from the learned factors and obtain the spectral embedding through a reduced eigenproblem. Both phases operate directly on the factors, which makes their computational cost linear.

Let the matrix $X=[x_1,\ldots,x_n]\in\mathbb R^{d_x\times n}$ contain the input samples $x_i \in \mathbb R^{d_x}$, and let $k$ denote the number of clusters.

\subsection{A Projector Approximation Motivated by LSR}
\label{sec:loom-geometry}

LSR~\cite{paper9} models each sample through a linear combination of the data collection and penalizes the energy of the coefficients. For the input $X$, its diagonal-unconstrained formulation is
\begin{equation}
 \min_{C\in\mathbb R^{n\times n}}\
 \frac12\|X-XC\|_F^2+\lambda\|C\|_F^2,
 \qquad \lambda>0.
 \label{eq:loom-ridge}
\end{equation}
The reconstruction term encourages self-expression. The Frobenius penalty controls coefficient magnitudes while allowing correlated samples to contribute jointly. Differentiation with respect to $C$ gives the closed-form solution of this problem
\begin{equation}
 C^\star_{\mathrm{LSR}}
 =(X^\top X+2\lambda I_n)^{-1}X^\top X,
 \label{eq:loom-ridge-solution}
\end{equation}
where \(I_n\) is the \(n\times n\) identity matrix. A direct evaluation of Eq.~\eqref{eq:loom-ridge-solution} requires forming the Gram matrix \(X^\top X\), at a cost of \(O(d_xn^2)\), followed by the solution of an \(n\times n\) linear system, which requires \(O(n^3)\) operations using a dense factorization. Moreover, the resulting coefficient matrix contains \(n^2\) entries and therefore requires \(O(n^2)\) memory. Although the inversion cost can be reduced in particular dimensional regimes, explicitly constructing the dense self-representation remains quadratic in the number of samples. Thereofre, we examine the spectral structure of \(C^\star_{\mathrm{LSR}}\) to derive a compact representation.

Let \(r=\operatorname{rank}(X)\). Since \(X^\top X\) is symmetric positive semidefinite with rank \(r\), its compact eigendecomposition is
\begin{equation}
X^\top X = V_x\operatorname{diag}(\beta_1,\ldots,\beta_r)V_x^\top,
\label{eq:loom-input-gram-eigendecomposition}
\end{equation}
where \(V_x\in\mathbb R^{n\times r}\) has orthonormal columns given by the eigenvectors of \(X^\top X\) associated with its positive eigenvalues, and \(\beta_1\ge\cdots\ge\beta_r>0\) are the corresponding eigenvalues.

\begin{proposition}[Spectral structure of LSR]
\label{prop:loom-lsr-spectrum}
For $\lambda>0$, the LSR solution is symmetric and positive semidefinite, with rank $r$. Its nonzero eigenvalues are
\begin{equation}
 \tau_j=\frac{\beta_j}{\beta_j+2\lambda},
 \qquad j=\left\{1,\ldots,r\right\},
 \label{eq:loom-lsr-spectrum}
\end{equation}
and $C^\star_{\mathrm{LSR}}=V_x\operatorname{diag}(\tau_1,\ldots,\tau_{r})V_x^\top$.
\end{proposition}

The eigenvectors identify the sample relationships represented by LSR. Eq.~\eqref{eq:loom-lsr-spectrum} also shows that directions with $\beta_j$ large relative to $2\lambda$ have eigenvalues close to one. This observation motivates replacing their spectral weights by one and retaining the corresponding subspace. For the full row space, define $C_{\mathrm{row}}=V_xV_x^\top\in\mathbb R^{n\times n}$. The approximation error is
\begin{equation}
 \|C^\star_{\mathrm{LSR}}-C_{\mathrm{row}}\|_F^2
 =\sum_{j=1}^{r}
 \left(\frac{2\lambda}{\beta_j+2\lambda}\right)^2.
 \label{eq:loom-ridge-gap}
\end{equation}
Thus, the projector approaches the LSR representation as the regularization becomes small relative to the positive data spectrum. This gives a direct spectral justification for an idempotent self-representation.

To obtain a compact model, choose a factor width $m$ with $k\le m<n$. When $m\le r$, let $V_{x,m}\in\mathbb R^{n\times m}$ contain the first $m$ columns of $V_x$. The corresponding approximation satisfies

\begin{equation}
\|C^\star_{\mathrm{LSR}}-V_{x,m}V_{x,m}^\top\|_F^2=\sum_{j=1}^{m}(1-\tau_j)^2 +\sum_{j=m+1}^{r}\tau_j^2.
\label{eq:loom-truncated-approximation}
\end{equation}

For the retained directions \(j\le m\), the approximation replaces the LSR eigenvalue \(\tau_j\) by the projector eigenvalue \(1\). The first term in Eq.~\eqref{eq:loom-truncated-approximation} measures the resulting error. The remaining directions \(j>m\) are discarded, and the second term measures the corresponding loss of the LSR spectral components. Since $\tau_j$ is monotonically increasing with \(\beta_j\), the ordering of the \(\tau_j\)'s is identical to that of the positive eigenvalues of \(X^\top X\). Hence, for a fixed width \(m\), retaining the first \(m\) eigenvectors gives the closest projector among the rank-\(m\) projectors formed from the eigenvectors of \(X^\top X\).

Setting \(m=r\) recovers the full row-space projector \(C_{\mathrm{row}}=V_xV_x^\top\). However, constructing this projector requires determining \(r\) and computing all \(r\) nonzero right singular vectors of \(X\). For a dense \(X\in\mathbb R^{d_x\times n}\), computing the full singular subspace costs \(O(d_xn\min\{d_x,n\})\), while storing \(V_x\) requires \(O(nr)\) memory. These costs become prohibitive when \(r\) grows with the data size. Thus, LoomSC does not attempt to recover the complete rank-\(r\) projector. Instead, \(m\) is treated as a prescribed model width, and the corresponding thin factors are learned directly in the subsequent formulation. %

Accordingly, we parameterize the self-representation as
\begin{equation}
 C=PP^\top,\qquad
 P\in\mathbb R^{n\times m},\qquad P^\top P=I_m.
 \label{eq:loom-projector}
\end{equation}
The factor $P$ stores the relationships between all samples through $nm$ coefficients. Its orthogonality makes $C$ symmetric, positive semidefinite, and idempotent, with rank $m$. It also fixes the coefficient energy:
\begin{equation}
 \|PP^\top\|_F^2
 =\operatorname{tr}(P^\top P)=m.
 \label{eq:loom-constant}
\end{equation}

Substituting Eq.~\eqref{eq:loom-projector} into the LSR objective in Eq.~\eqref{eq:loom-ridge} yields the following optimization over fixed-width projectors:
\begin{equation}
  \min_{P^\top P=I_m}\ \|X-XPP^\top\|_F^2.
  \label{eq:loom-projector-objective}
\end{equation}
Since $PP^\top$ is an orthogonal projector, Eq.~\eqref{eq:loom-projector-objective} is equivalent to maximizing $\operatorname{tr}(P^\top X^\top XP)$. Hence, for $m\le r$, an optimal solution is $P^\star=V_{x,m}R$, %
where $R\in\mathbb R^{m\times m}$ is any orthogonal matrix. All such solutions span the same dominant $m$-dimensional subspace and yield
\begin{equation}
P^\star P^{\star\top}=V_{x,m}V_{x,m}^\top.
\label{eq:loom-projector-solution}
\end{equation}

The minimum reconstruction error is $\sum_{j=m+1}^{r}\beta_j$. Thus, the exact solution requires computing the dominant right singular subspace of the current representation. This is practical when the representation is fixed, but it becomes less attractive when the features are learned jointly, since the dominant subspace must be recomputed as the representation evolves. To address this problem, LoomSC replaces the repeated spectral solution with a two-factor formulation. The resulting updates involve only thin matrix products and an $n\times m$ Procrustes problem, while the full coefficient matrix $C$ remains implicit.

\subsection{Learning a Latent Representation}
\label{sec:loom-model}

We apply the compact projector model to a learned latent space. This lets the self-expression structure guide feature extraction. We use an encoder $\mathcal E$ and a decoder $\mathcal D$ with parameter collections $W$ and $\widehat W$, respectively. They define
\begin{equation}
 \begin{aligned}
 Z&=\mathcal E(X;W)\in\mathbb R^{d_z\times n},\\
 \widehat X&=\mathcal D(Z;\widehat W)\in\mathbb R^{d_x\times n},
 \end{aligned}
 \label{eq:loom-encoding}
\end{equation}
where $d_z$ is the latent dimension and $z_i\in\mathbb R^{d_z}$ is the representation of $x_i$. The vector $\widehat x_i\in\mathbb R^{d_x}$ is the $i^{\text{th}}$ column  of the matrix $\widehat X$. For image inputs, convolutional layers extract spatial features, which are flattened to form the columns of the latent representation matrix $Z$. The decoder reconstructs the input from these features. Its reconstruction objective is
\begin{equation}
 J_{\mathrm{rec}}(W,\widehat W)
 =\frac1n\|X-\widehat X\|_F^2.
 \label{eq:loom-pretrain}
\end{equation}

This objective trains the representation to retain information about the observed samples. The projector objective in Eq.~\eqref{eq:loom-projector-objective}, applied to $Z$, imposes a complementary subspace structure.

\subsection{Self-Expression Through Two Thin Factors}
\label{sec:loom-factorization}

To optimize projector self-expression in the latent space, we introduce a latent factor $L\in\mathbb R^{d_z\times m}$. Its columns represent learned directions in the latent space. The following decomposition connects the resulting two-factor model to projector self-expression:
\begin{equation}
 \|Z-LP^\top\|_F^2
 =\|Z-ZPP^\top\|_F^2+\|L-ZP\|_F^2.
 \label{eq:loom-split}
\end{equation}

The identity in Eq.~\eqref{eq:loom-split} holds for every $P$ satisfying $P^\top P=I_m$. The proof is provided in Appendix~B. Minimizing its left-hand side with respect to $L$ gives $L=ZP$. Hence, we obtain:
\begin{equation}
 \min_{L,\,P^\top P=I_m}\|Z-LP^\top\|_F^2
 =\min_{P^\top P=I_m}\|Z-ZPP^\top\|_F^2.
 \label{eq:loom-factor-objective}
\end{equation}

The two-factor formulation separates the latent directions represented by $L$ from the sample coefficients encoded by $P$. This separation gives simple block updates and requires only products involving thin matrices.

For fixed $Z$, let $r_z=\operatorname{rank}(Z)$ and write $Z=U_z\Sigma_zV_z^\top$. The matrices $U_z\in\mathbb R^{d_z\times r_z}$ and $V_z\in\mathbb R^{n\times r_z}$ have orthonormal columns. The diagonal matrix $\Sigma_z\in\mathbb R^{r_z\times r_z}$ contains the positive singular values $\sigma_1\ge\cdots\ge\sigma_{r_z}>0$.

\begin{proposition}[Reconstruction through dominant directions]
\label{prop:loom-factor}
For a fixed latent matrix $Z$, the optimal factor residual is
\begin{equation}
 \min_{L,\,P^\top P=I_m}\|Z-LP^\top\|_F^2
 =\sum_{j=m+1}^{r_z}\sigma_j^2,
 \label{eq:loom-tail}
\end{equation}
where an empty sum is zero. For $m\le r_z$, a minimizing $P$ spans an $m$-dimensional dominant right singular subspace of $Z$. For $m\ge r_z$, an optimal projector contains the row space of $Z$ and satisfies $ZPP^\top=Z$. In particular, $m=r_z$ gives $C=V_zV_z^\top$.
\end{proposition}

Proposition~\ref{prop:loom-factor} connects the learned factors to the dominant spectral directions of $Z$. Thus, the spectral approximation developed above applies directly to the learned representation. It also quantifies how the width controls the amount of latent information represented by the model. The proofs of~\eqref{eq:loom-split} and Proposition~\ref{prop:loom-factor} are given in Appendix~\ref{app:loom-factor}.

We couple the factor objective with the auto-encoder to learn the latent features and their self-representation jointly:
\begin{equation}
 \begin{aligned}
 \min_{W,\widehat W,L,P}\quad
 J(W,\widehat W,L,P)
 &=J_{\mathrm{rec}}(W,\widehat W)+\|Z-LP^\top\|_F^2,\\
 \text{subject to}\quad P^\top P&=I_m.
 \end{aligned}
 \label{eq:loom-joint}
\end{equation}

The reconstruction term retains information about the input samples. The factor term encourages the latent codes to admit a common compact self-representation. Joint optimization allows the encoder to adjust this representation as the factors are refined. For sample $i$, the row coefficients $p_i=P_{i:}^\top\in\mathbb R^m$ reconstruct its latent vector as $Lp_i$.

\subsection{Alternating Optimization}
\label{sec:loom-optimization}

We first pretrain the auto-encoder with~\eqref{eq:loom-pretrain}. We then initialize the columns of $L$ using $m$ latent samples selected by $k$-means++ seeding~\cite{loom_kmeanspp}. This spreads the initial directions across the observed latent features. Appendix~E specifies the sampling rule. The factor $P$ is initialized by its update below. Each joint iteration updates the network parameters, refreshes $Z$, and then updates $P$ and $L$ in sequence.

\paragraph{Network parameters}
With $L$ and $P$ fixed, an Adam step~\cite{loom_adam} updates the encoder and decoder:
\begin{equation}
 (W,\widehat W)\leftarrow
 \operatorname{Adam}_{t}\!\left(
 (W,\widehat W),\nabla_{W,\widehat W}J\right),
 \label{eq:loom-network}
\end{equation}
where $t$ is the joint iteration index. The factor term contributes the latent gradient $2(Z-LP^\top)$, which encourages the encoder to produce representations consistent with the current factors. To control activation memory, the samples can be processed in chunks $\{\mathcal B_b\}$ while keeping $L$ and $P$ fixed:
\begin{equation}
 J=\sum_b\sum_{i\in\mathcal B_b}
 \left(
 \frac1n\|x_i-\widehat x_i\|_2^2
 +\|z_i-Lp_i\|_2^2
 \right).
 \label{eq:loom-batches}
\end{equation}
The gradients are accumulated over all chunks before a single update of $(W,\widehat W)$. Hence, chunking reduces memory usage without changing the full-data objective. After the update, the encoder recomputes $Z$ before updating $P$ and $L$.
\paragraph{Sample factor}
For fixed $Z$ and $L$, orthogonality gives
\begin{equation}
 \|Z-LP^\top\|_F^2
 =\|Z\|_F^2+\|L\|_F^2
 -2\operatorname{tr}(P^\top Z^\top L).
 \label{eq:loom-procrustes-objective}
\end{equation}
Thus, updating $P$ amounts to maximizing the trace term under $P^\top P=I_m$. This is an orthogonal Procrustes problem. To solve it, we form the thin product $Z^\top L\in\mathbb R^{n\times m}$ and compute
\begin{equation}
 Z^\top L=U_P\Sigma_PV_P^\top,
 \qquad P\leftarrow U_PV_P^\top,
 \label{eq:loom-p-update}
\end{equation}
where $U_P\in\mathbb R^{n\times m}$ and $V_P\in\mathbb R^{m\times m}$ have orthonormal columns. The diagonal matrix $\Sigma_P\in\mathbb R^{m\times m}$ contains the singular values. This update attains the maximum trace and preserves the orthogonality constraint.

\paragraph{Latent factor}
For fixed $Z$ and $P$, differentiation gives $\nabla_L\|Z-LP^\top\|_F^2=2(L-ZP)$. The minimizing update is
\begin{equation}
 L\leftarrow ZP.
 \label{eq:loom-l-update}
\end{equation}
Each column of $L$ is refined using the current latent samples and their coefficients. Together, the two updates adapt the representation to the full dataset using an $n\times m$ decomposition and thin matrix products.

\begin{proposition}[Descent of the factor updates]
\label{prop:loom-descent}
For fixed $Z$, each update in~\eqref{eq:loom-p-update} and~\eqref{eq:loom-l-update} minimizes the factor objective with respect to its block. Repeating these updates yields a nonincreasing sequence of objective values that converges to a finite limit.
\end{proposition}
Appendix~\ref{app:loom-factor} proves this result. Algorithm~\ref{alg:loom-main} combines the factor updates with representation learning for a prescribed number of joint iterations.

\begin{algorithm}[!t]
\caption{LoomSC}
\label{alg:loom-main}
\begin{algorithmic}[1]
\Require $X\in\mathbb R^{d_x\times n}$, $k\leq m<n$, $T_{\mathrm{pre}},T$.
\Ensure Cluster labels $\gamma_1,\ldots,\gamma_n$.
\State Pretrain $(W,\widehat W)$ for $T_{\mathrm{pre}}$ updates using \eqref{eq:loom-pretrain}.
\State Compute $Z$ using \eqref{eq:loom-encoding}.
\State Initialize $L$ with $m$ sample columns selected by applying $k$-means++ to $Z$.
\State Initialize $P$ using \eqref{eq:loom-p-update}.
\For{$t=1,\ldots,T$}
  \State Fix $L$ and $P$ and accumulate the gradient using \eqref{eq:loom-batches}.
  \State Update $(W,\widehat W)$ using \eqref{eq:loom-network}.
  \State Recompute $Z$ using \eqref{eq:loom-encoding}.
  \State Update $P$ using the thin-SVD step in \eqref{eq:loom-p-update}.
  \State Update $L\leftarrow ZP$ using \eqref{eq:loom-l-update}.
\EndFor
\State Obtain the cluster labels from Algorithm~\ref{alg:loom-spectral} using the final $P$ and $Z$.
\State \Return $\gamma_1,\ldots,\gamma_n$.
\end{algorithmic}
\end{algorithm}

\subsection{Subspace Structure of the Learned Projector}
\label{sec:loom-subspaces}

The projector interpretation also explains how self-expression can preserve subspace memberships. Consider latent samples spanning subspaces $S_1,\ldots,S_k\subseteq\mathbb R^{d_z}$ with dimensions $r_1,\ldots,r_k$. For independent subspaces, their sum is direct and
\begin{equation}
 \dim(S_1+\cdots+S_k)=\sum_{a=1}^k r_a=r_z.
 \label{eq:loom-direct-sum}
\end{equation}

\begin{proposition}[Subspace preservation and connectivity]
\label{prop:loom-recovery}
Suppose all latent samples are nonzero and span the direct sum in~\eqref{eq:loom-direct-sum}. If $m=r_z<n$ and the fixed-$Z$ factor objective attains its minimum, then $C=PP^\top$ is block diagonal after grouping samples by subspace. Within a subspace, its nonzero off-diagonal coefficients define a connected graph whenever that subspace's samples admit no partition into two nonempty sets whose spans form a direct sum.
\end{proposition}

The result identifies conditions under which the projector both separates the generating subspaces and connects the samples within each subspace. Its proof uses the block structure of the row-space projector and appears in Appendix~\ref{app:loom-recovery}. The next phase converts these coefficients into nonnegative graph weights while preserving their nonzero connections.

\subsection{Scalable Spectral Clustering from the Learned Factors}
\label{subsec:loom-spectral}

The learned factor $P$ provides the pairwise coefficients through $C=PP^\top$. We now use this structure to obtain the final partition directly from compact matrices. The construction has two steps. First, we convert the coefficients into nonnegative affinities while preserving their support. We then express the normalized spectral problem through a quadratic feature map whose dimension depends only on $m$.

\paragraph{A nonnegative affinity with an explicit feature map}
The magnitude of $C_{ij}=p_i^\top p_j$ measures the strength of the relation between samples $i$ and $j$. Squaring this coefficient gives the symmetric affinity
\begin{equation}
 A=C\circ C\in\mathbb R^{n\times n},\qquad
 A_{ij}=(p_i^\top p_j)^2,
 \label{eq:loom-affinity}
\end{equation}
where $\circ$ denotes entrywise multiplication. This construction assigns nonnegative weights to both positive and negative coefficients. It also preserves their zero pattern. Thus, any separation between subspaces encoded by $C$ is retained in $A$. Moreover, the affinity is invariant to an orthogonal change of basis in $P$.

The quadratic form in~\eqref{eq:loom-affinity} admits an explicit feature representation. Let $q=m(m+1)/2$. For a symmetric matrix $S\in\mathbb R^{m\times m}$, $\operatorname{svec}(S)\in\mathbb R^q$ contains its $m$ diagonal entries followed by its upper-triangular entries. Each off-diagonal entry is multiplied by $\sqrt{2}$. This scaling preserves the Frobenius inner product between symmetric matrices. We define
\begin{equation}
 \begin{aligned}
 \phi(p)&=\operatorname{svec}(pp^\top)\in\mathbb R^q,\\
 F_{i:}&=\phi(p_i)^\top,\qquad F\in\mathbb R^{n\times q}.
 \end{aligned}
 \label{eq:loom-lift}
\end{equation}
Consequently, $\phi(p_i)^\top\phi(p_j)=(p_i^\top p_j)^2$, which gives
\begin{equation}
 A=FF^\top.
 \label{eq:loom-affinity-factor}
\end{equation}
Hence, the $n\times n$ affinity matrix is represented exactly through $q$ features per sample and does not need to be formed explicitly.

We retain the diagonal entries of $A$ as self-loop weights. The projector identity $C^2=C$ then gives a particularly simple degree formula,
\begin{equation}
 \delta_i=\sum_{j=1}^n A_{ij}
 =(C^2)_{ii}=C_{ii}=\|p_i\|_2^2,
 \qquad \sum_{i=1}^n\delta_i=m.
 \label{eq:loom-degree}
\end{equation}
Each degree therefore requires only a row norm of $P$. Computing all degrees takes $O(nm)$ time. Let $\mathcal I_+=\{i:\delta_i>0\}$ denote the indices with positive degrees and let $n_+=|\mathcal I_+|$. Spectral normalization is performed on these samples. Zero-degree samples receive labels through their latent features after clustering.

\paragraph{An exact reduced spectral problem}
Let $F_+\in\mathbb R^{n_+\times q}$ contain the rows indexed by $\mathcal I_+$ in increasing order. Define
\begin{equation}
 \begin{aligned}
 \Delta&=\operatorname{diag}(\delta_i:i\in\mathcal I_+)
       \in\mathbb R^{n_+\times n_+},\\
 B&=\Delta^{-1/2}F_+\in\mathbb R^{n_+\times q},\\
 G&=B^\top B\in\mathbb R^{q\times q}.
 \end{aligned}
 \label{eq:loom-normalization}
\end{equation}
The normalized affinity is $BB^\top$, and its normalized Laplacian is $I_{n_+}-BB^\top$. Thus, the smallest Laplacian eigenvalues correspond to the largest eigenvalues of $BB^\top$. The following result transfers their computation to $G$.

\begin{proposition}[Exact spectral reduction]
\label{prop:loom-spectral}
For $P^\top P=I_m$, the normalized feature matrix satisfies
\begin{equation}
 m\leq\operatorname{rank}(B)\leq\min\{n_+,q\}.
 \label{eq:loom-lift-rank}
\end{equation}
The positive eigenvalues of $G$ and $BB^\top$ coincide, including multiplicities. If $Gv_j=\nu_jv_j$ with $\nu_j>0$ and orthonormal $v_j\in\mathbb R^q$, then $Bv_j/\sqrt{\nu_j}$ are orthonormal eigenvectors of $BB^\top$ with eigenvalues $\nu_j$.
\end{proposition}

The proof is provided in Appendix~\ref{app:loom-spectral-proofs}. Since $k\leq m$, the reduced problem supplies at least $k$ positive eigenvalues. The embedding dimension is $k$, whereas $m$ controls the factor width and $q$ determines the reduced spectral dimension.

The degree identity also gives one dominant eigenpair analytically,
\begin{equation}
 v_1=\frac{\operatorname{svec}(I_m)}{\sqrt m},\qquad
 \nu_1=1,\qquad
 Bv_1=\frac{(\sqrt{\delta_i})_{i\in\mathcal I_+}}{\sqrt m}.
 \label{eq:loom-leading-vector}
\end{equation}
We retain this vector and compute the largest $k-1$ eigenpairs on $v_1^\perp$. This gives $k$ dominant modes while providing a strictly positive coordinate for every active sample. The sample embedding and its normalized rows are
\begin{equation}
 \begin{aligned}
 H&=B[v_1,\ldots,v_k]
       \operatorname{diag}(\nu_1,\ldots,\nu_k)^{-1/2}
       \in\mathbb R^{n_+\times k},\\
 \widehat H_{i:}&=H_{i:}/\|H_{i:}\|_2,\qquad i=1,\ldots,n_+.
 \end{aligned}
 \label{eq:loom-spectral-embedding}
\end{equation}
Equation~\eqref{eq:loom-leading-vector} ensures that every row norm is positive. Retaining the eigenvalue-one modes also preserves the component indicators when the affinity graph has $k$ connected components. In that case, row normalization maps each component to a distinct embedding point, as established in Appendix~\ref{app:loom-spectral-proofs}.

Following normalized spectral clustering~\cite{loom_njw}, we apply $k$-means to the rows of $\widehat H$. Let $\gamma_i\in\{1,\ldots,k\}$ denote the resulting label of active sample $i$. For a zero-degree sample, we use its distance to the latent centers of these clusters,
\begin{equation}
 \begin{aligned}
 \mathcal J_a&=\{i\in\mathcal I_+:\gamma_i=a\},\\
 \zeta_a&=\frac{1}{|\mathcal J_a|}\sum_{i\in\mathcal J_a}z_i
           \in\mathbb R^{d_z},\\
 \gamma_i&\in\arg\min_{a\in\{1,\ldots,k\}}
          \|z_i-\zeta_a\|_2^2,\qquad i\notin\mathcal I_+.
 \end{aligned}
 \label{eq:loom-isolate-assignment}
\end{equation}
The clustering conventions in Appendix~\ref{app:loom-implementation} ensure nonempty active clusters and specify initialization and tie-breaking. Together, these steps assign every sample to a cluster. The corresponding membership matrix is $\Gamma\in\{0,1\}^{n\times k}$ with $\Gamma_{ia}=\mathbf1\{\gamma_i=a\}$.

\paragraph{Direct and implicit implementations}
When $q$ is modest, we accumulate $G$ from normalized feature rows and solve its symmetric eigenproblem directly. For larger $q$, the same operator can be applied by streaming the rows of $P$. Let $\operatorname{smat}:\mathbb R^q\to\mathbb R^{m\times m}$ be the inverse of $\operatorname{svec}$ on symmetric matrices. For any $v\in\mathbb R^q$,
\begin{equation}
 Gv=\operatorname{svec}\!\left(
  \sum_{i\in\mathcal I_+}
  \frac{p_i^\top\operatorname{smat}(v)p_i}{\delta_i}
       p_ip_i^\top\right).
 \label{eq:loom-gram-action}
\end{equation}
Each operator evaluation takes $O(nm^2)$ time. A symmetric partial eigensolver uses this action on $v_1^\perp$ to recover the remaining dominant modes. Both implementations solve the spectral problem defined by~\eqref{eq:loom-affinity}. They use compact factors throughout, with $C$, $A$, and the sample Laplacian represented implicitly.

\begin{algorithm}[!t]
\caption{Spectral clustering through the quadratic feature map}
\label{alg:loom-spectral}
\begin{algorithmic}[1]
\Require $P\in\mathbb R^{n\times m}$ with $P^\top P=I_m$, $Z\in\mathbb R^{d_z\times n}$, $k\leq m$, $R_{\mathrm{km}},T_{\mathrm{km}}$.
\Ensure Cluster labels $\gamma_1,\ldots,\gamma_n$.
\State Compute $\delta_i$ by \eqref{eq:loom-degree} and retain the indices $\mathcal I_+$ with $\delta_i>0$.
\State Set $q=m(m+1)/2$ and the feature ordering in \eqref{eq:loom-lift}.
\State Set $v_1,\nu_1$ using \eqref{eq:loom-leading-vector}.
\If{$k>1$}
  \If{direct eigensolver}
    \State Accumulate $G=\sum_{i\in\mathcal I_+}\phi(p_i)\phi(p_i)^\top/\delta_i$ using \eqref{eq:loom-normalization}.
    \State Compute its largest $k-1$ eigenpairs on $v_1^\perp$.
  \Else
    \State Compute the largest $k-1$ eigenpairs on $v_1^\perp$ with a symmetric partial eigensolver using \eqref{eq:loom-gram-action}.
  \EndIf
\EndIf
\State Recover $H$ and normalize its rows by \eqref{eq:loom-spectral-embedding}.
\State Run $R_{\mathrm{km}}$ seeded $k$-means restarts using \eqref{eq:loom-kmeans}, with at most $T_{\mathrm{km}}$ iterations each.
\State Retain the restart with minimum distortion and restore the original sample indices.
\State Assign zero-degree samples by \eqref{eq:loom-isolate-assignment}.
\State \Return $\gamma_1,\ldots,\gamma_n$.
\end{algorithmic}
\end{algorithm}

\subsection{Complete Algorithm and Computational Complexity}
\label{sec:loom-complexity}

Algorithm~\ref{alg:loom-main} combines representation learning with spectral clustering. Let $T_{\mathrm{pre}}$ and $T$ be the numbers of pretraining updates and joint outer iterations, respectively. The spectral routine is given in Algorithm~\ref{alg:loom-spectral}. It uses $R_{\mathrm{km}}$ clustering restarts with at most $T_{\mathrm{km}}$ Lloyd iterations per restart.

\paragraph{Representation learning}
Let $c_{\mathrm{net}}$ bound the per-sample cost of network evaluation and backpropagation in one update, including the refreshed encoder evaluation. Pretraining takes $O(T_{\mathrm{pre}}nc_{\mathrm{net}})$ time. Seeding requires $O(nd_zm)$ time, and the initial Procrustes update requires $O(nd_zm+nm^2+m^3)$. Each joint iteration takes
\begin{equation}
 O\!\left(nc_{\mathrm{net}}+nd_zm+nm^2+m^3\right).
 \label{eq:loom-training-cost}
\end{equation}
These terms account for network training, the products $Z^\top L$ and $ZP$, and the thin SVD. The representation cost is the sum of pretraining and initialization costs plus $T$ times~\eqref{eq:loom-training-cost}. Gradient accumulation allows the activation workspace to be controlled through the chunk size.

\paragraph{Spectral computation}
Computing the degrees requires $O(nm)$ time. In the direct implementation, constructing the reduced matrix $G\in\mathbb R^{q\times q}$ costs $O(nq^2)$, and its eigendecomposition costs $O(q^3)$. Recovering the $k$-dimensional sample embedding requires $O(nqk)$ operations. The total complexity is therefore
\begin{equation}
 O(nm+nq^2+q^3+nqk),\qquad q=m(m+1)/2.
 \label{eq:loom-direct-cost}
\end{equation}
The dependence on $n$ appears only in the construction of the reduced matrix and the recovery of the sample embedding. The eigendecomposition is performed in the $q$-dimensional space.

For the implicit implementation, the matrix $G$ is not formed explicitly. Let $T_{\mathrm{eig}}$ denote the total number of matrix--vector products required by the iterative eigensolver, $s_{\mathrm{eig}}$ the maximum Krylov basis size, and $R_{\mathrm{eig}}$ the number of restarts. Including operator evaluations, orthogonalization, and the projected eigenproblems, the total complexity is
\begin{equation}
 O\!\left(
 nm
 +T_{\mathrm{eig}}(nq+qs_{\mathrm{eig}})
 +R_{\mathrm{eig}}(qs_{\mathrm{eig}}^2+s_{\mathrm{eig}}^3)
 +nqk
 \right).
 \label{eq:loom-iterative-cost}
\end{equation}
Each vector processed by a block eigensolver contributes one matrix--vector product to $T_{\mathrm{eig}}$. The quantities $T_{\mathrm{eig}}$, $s_{\mathrm{eig}}$, and $R_{\mathrm{eig}}$ depend on the convergence of the eigensolver and the required numerical accuracy. Since the implicit implementation evaluates products with $G$ only when needed, it avoids storing the dense $q\times q$ matrix and is preferable when $q$ is large.

\paragraph{Clustering and memory}
The $k$-dimensional embedding gives a clustering cost of $O(R_{\mathrm{km}}(T_{\mathrm{km}}+1)nk^2)$. Computing latent centers and assigning zero-degree samples adds at most $O(nd_zk)$. Let $M_{\mathrm{net}}$ include network parameters, optimizer state and the activation workspace. A resident-data memory bound is
\begin{equation}
 O\!\left(n(d_x+d_z+m+k)+d_zm+M_{\mathrm{net}}\right),
 \label{eq:loom-base-memory}
\end{equation}
plus $O(q^2+qk)$ for direct reduction or $O(qs_{\mathrm{eig}}+qk)$ for the implicit eigensolver. Quadratic feature rows can be evaluated as needed, so storing $F$ is optional. The output labels require $O(n)$ memory.

Together, these bounds establish linear time and memory in $n$ for fixed dimensions, architecture, factor width, cluster count and iteration budgets. The factor width $m$ is a chosen model parameter, so both phases operate without estimating the data rank. Their shared use of compact factors makes scalability extend from representation learning to the final partition.

\section{Results and Discussion}
\label{sec:experiments}

\paragraph{Protocol}
We evaluate LoomSC on five image-clustering benchmarks. The face datasets are Extended YaleB (EYaleB), ORL, and UMIST, with $(n,k)=(2{,}432,38)$, $(400,40)$, and $(575,20)$, respectively. We also use COIL-100 with $7{,}200$ images from $100$ object classes and a balanced subset of Fashion-MNIST with $5{,}000$ images from $10$ classes. Each sample is normalized to unit $\ell_2$ norm. Tables~\ref{tab:acc} and~\ref{tab:time} specify the sample counts used in the comparisons.

We compare with nine baselines that distinguish deep representation learning from computational scalability. The \emph{scalable non-deep methods} comprise two groups. LMVSC~\cite{paper15}, SGL~\cite{paper36}, and S$^5$C~\cite{loom_s5c} use sampling or anchor-based reduction. EnSC~\cite{paper19}, SSC-OMP~\cite{paper17}, A-DSSC~\cite{paper22}, and SSSC~\cite{paper18} accelerate self-expression estimation without learning deep representations. These methods target scalability, but do not all achieve linear complexity for the complete clustering pipeline.

The \emph{non-linear-time full-affinity references} are LSR~\cite{paper9} and PRO-DSC~\cite{loom_pro_dsc}. LSR is non-deep and computes a dense self-expression matrix. PRO-DSC is deep and supports mini-batch training. Its published spectral inference nevertheless constructs pairwise coefficients over the clustering set and processes the resulting affinity. Thus, scalable training does not eliminate its full-affinity spectral bottleneck. LoomSC combines deep representation learning with factorized self-expression and spectral assignment. Its complete pipeline has linear time and memory complexity in $n$ for fixed architecture, dimensions, factor width, and iteration budgets.

A detailed comparison of computational and matrix-storage complexities is provided in Appendix~\ref{app:complexity-comparison}.

For LoomSC, we use the same architecture and optimization hyperparameters across all datasets. The encoder contains $L=3$ convolutional layers with $10$, $20$, and $30$ output channels and kernel sizes $5\times5$, $3\times3$, and $3\times3$, respectively. Each layer uses stride $2$ and a ReLU activation, with no batch normalization. The decoder mirrors the encoder using transposed convolutions. The resulting convolutional representation has size $30\times6\times6$, which gives a flattened latent dimension of $d_z=1080$. We set the factor width to $m=380$ for all datasets. The auto-encoder is first pretrained for $5{,}000$ epochs using full-batch Adam with a learning rate of $10^{-3}$. No data augmentation, weight decay, or learning-rate scheduling is used.

After pretraining, $k$-means with $k$-means++ initialization is applied to the latent representations, and its $m$ centroids initialize the columns of $L$. The factor $P$ is then initialized by the orthogonal Procrustes solution. Joint optimization is performed for another $5{,}000$ epochs using full-batch Adam with a learning rate of $10^{-3}$. The reconstruction and self-expression terms have equal weights. At each joint-training epoch, the network parameters are first updated by Adam. The factor $P$ is then updated by the exact Procrustes solution, followed by the closed-form update $L=ZP$, as described in Algorithm~\ref{alg:loom-main}. After joint optimization, spectral assignment is performed following Algorithm~\ref{alg:loom-spectral}. The spectral embedding dimension is set to the target number of clusters $k$, after which $k$-means produces the final partition.

We report clustering accuracy (ACC), normalized mutual information (NMI), and adjusted Rand index (ARI) as percentages. Higher values indicate better clustering. Repeated-run results are summarized by their mean and standard deviation over three seeds. A scalar entry indicates that no standard deviation is reported. Differences between percentage scores are expressed in percentage points. The synthetic experiments additionally report the connectivity diagnostic (CONN)~\cite{chen2020stochastic}, for which higher values are preferred.

\subsection{Clustering quality}

Tables~\ref{tab:acc}--\ref{tab:ari} show that LoomSC ranks first or second in all $15$ dataset--metric comparisons, including a shared second place for COIL-100 NMI. No other evaluated method maintains this ranking across all five datasets and three metrics. Its unweighted means are $80.17\%$ ACC, $85.98\%$ NMI, and $72.77\%$ ARI. These exceed the highest corresponding baseline means by $6.66$, $4.09$, and $10.02$ points. Thus, the compact formulation provides consistently strong clustering quality while retaining linear-in-$n$ computation under the conditions stated above.

\begin{table*}[t]
\centering
\setlength{\tabcolsep}{3pt}
\caption{Clustering quality --- ACC (\%) on the five benchmarks. Higher is better. Mean $\pm$ standard deviation over three seeds where available. Bold indicates the largest mean in each row.}
\label{tab:acc}
\resizebox{\textwidth}{!}{%
\begin{tabular}{lcc*{10}{c}}
\toprule
\multicolumn{3}{c}{} & \multicolumn{8}{c}{Non-deep methods} & \multicolumn{2}{c}{Deep methods} \\
\cmidrule(lr){4-11}\cmidrule(lr){12-13}
Dataset & $n$ & $k$ & EnSC & SSC-OMP & LMVSC & SGL & S$^5$C & LSR & A-DSSC & SSSC & PRO-DSC & LoomSC \\
\midrule
Fashion-MNIST & 5,000 & 10 & 61.35 $\pm$ 0.00 & 56.19 $\pm$ 1.36 & 59.18 & 54.47 $\pm$ 0.36 & 54.62 $\pm$ 1.99 & 59.52 $\pm$ 0.00 & 39.01 $\pm$ 2.20 & 42.01 $\pm$ 1.70 & 52.12 $\pm$ 2.92 & \textbf{65.33 $\pm$ 1.79} \\
ORL & 400 & 40 & 78.00 $\pm$ 1.59 & 73.75 $\pm$ 0.82 & 55.25 $\pm$ 0.00 & 62.25 $\pm$ 0.00 & 42.50 $\pm$ 0.00 & 82.50 $\pm$ 0.00 & 50.75 $\pm$ 0.00 & 46.50 $\pm$ 1.54 & 52.00 $\pm$ 7.91 & \textbf{86.17 $\pm$ 2.01} \\
UMIST & 575 & 20 & 72.52 $\pm$ 0.00 & 67.70 $\pm$ 0.00 & 60.62 $\pm$ 0.00 & 65.00 $\pm$ 0.00  & 40.87 $\pm$ 0.00 & 58.26 $\pm$ 0.00 & 52.52 $\pm$ 0.00 & 21.16 $\pm$ 1.31 & \textbf{81.39 $\pm$ 1.58} & 80.00 $\pm$ 2.19 \\
EYaleB & 2,432 & 38 & 88.19 $\pm$ 0.00  & 80.84 $\pm$ 1.34 & 21.50 $\pm$ 0.00 & 21.50 $\pm$ 0.00  & 22.41 $\pm$ 0.00 & \textbf{96.26 $\pm$ 0.00} & 61.88 $\pm$ 0.00 & 25.11 $\pm$ 2.27 & 77.63 $\pm$ 0.73 & 94.73 $\pm$ 1.26 \\
COIL-100 & 7,200 & 100 & 67.50 $\pm$ 0.69 & 59.46 $\pm$ 0.00  & 49.50 $\pm$ 0.00  & 49.26 $\pm$ 0.00 & 42.04 $\pm$ 0.00 & 61.10 $\pm$ 0.00 & 48.57 $\pm$ 0.00 & 13.81 $\pm$ 1.31 & \textbf{82.31 $\pm$ 0.30} & 74.64 $\pm$ 1.34 \\
\bottomrule
\end{tabular}}
\end{table*}

\begin{table*}[t]
\centering
\setlength{\tabcolsep}{3pt}
\caption{Clustering quality --- NMI (\%) on the five benchmarks. Higher is better. Mean $\pm$ standard deviation over three seeds. Bold indicates the largest mean in each row.}
\label{tab:nmi}
\resizebox{\textwidth}{!}{%
\begin{tabular}{lccccccccccc}
\toprule
\multicolumn{1}{c}{} & \multicolumn{8}{c}{Non-deep methods} & \multicolumn{2}{c}{Deep methods} \\
\cmidrule(lr){2-9}\cmidrule(lr){10-11}
Dataset & EnSC & SSC-OMP & LMVSC & SGL & S$^5$C & LSR & A-DSSC & SSSC & PRO-DSC & LoomSC \\
\midrule
Fashion-MNIST & 60.29 $\pm$ 0.88 & 57.93 $\pm$ 0.41 & \textbf{64.37 $\pm$ 0.63} & 62.51 $\pm$ 0.42 & 54.82 $\pm$ 2.33 & 63.33 $\pm$ 0.00 & 53.83 $\pm$ 1.00 & 46.88 $\pm$ 0.45 & 56.45 $\pm$ 3.08 & 63.91 $\pm$ 0.82 \\
ORL & 89.25 $\pm$ 0.07 & 84.54 $\pm$ 0.37 & 74.53 $\pm$ 0.00 & 78.24 $\pm$ 0.00 & 64.76 $\pm$ 0.00 & 90.07 $\pm$ 0.00 & 77.93 $\pm$ 0.00 & 66.84 $\pm$ 1.34 & 71.71 $\pm$ 5.39 & \textbf{91.76 $\pm$ 0.88} \\
UMIST & 87.06 $\pm$ 0.00 & 63.61 $\pm$ 0.53 & 62.49 $\pm$ 0.00 & 56.84 $\pm$ 0.00 & 56.72 $\pm$ 0.00 & 75.40 $\pm$ 0.00 & 66.01 $\pm$ 0.00 & 32.89 $\pm$ 1.40 & \textbf{90.02 $\pm$ 0.20} & 87.83 $\pm$ 1.44 \\
EYaleB & 74.31 $\pm$ 0.03 & 83.17 $\pm$ 0.55 & 13.37 $\pm$ 0.00 & 6.44 $\pm$ 0.00 & 37.92 $\pm$ 0.00 & 95.44 $\pm$ 0.00 & 75.61 $\pm$ 0.00 & 34.24 $\pm$ 2.42 & 82.38 $\pm$ 0.91 & \textbf{95.47 $\pm$ 0.43} \\
COIL-100 & 90.95 $\pm$ 0.02 & 57.11 $\pm$ 0.16 & 74.69 $\pm$ 0.00 & 74.54 $\pm$ 0.00 & 68.03 $\pm$ 0.00 & 85.21 $\pm$ 0.00 & 74.31 $\pm$ 0.00 & 37.80 $\pm$ 1.08 & \textbf{95.75 $\pm$ 0.05} & 90.95 $\pm$ 0.41 \\
\bottomrule
\end{tabular}}
\end{table*}

\begin{table*}[t]
\centering
\setlength{\tabcolsep}{3pt}
\caption{Clustering quality --- ARI (\%) on the five benchmarks. Higher is better. Mean $\pm$ standard deviation over three seeds. Bold indicates the largest mean in each row.}
\label{tab:ari}
\resizebox{\textwidth}{!}{%
\begin{tabular}{lccccccccccc}
\toprule
\multicolumn{1}{c}{} & \multicolumn{8}{c}{Non-deep methods} & \multicolumn{2}{c}{Deep methods} \\
\cmidrule(lr){2-9}\cmidrule(lr){10-11}
Dataset & EnSC & SSC-OMP & LMVSC & SGL & S$^5$C & LSR & A-DSSC & SSSC & PRO-DSC & LoomSC \\
\midrule
Fashion-MNIST & 43.21 $\pm$ 1.04 & 42.00 $\pm$ 0.16 & 45.35 $\pm$ 0.87 & 42.75 $\pm$ 0.22 & 40.27 $\pm$ 1.50 & 46.53 $\pm$ 0.00 & 22.24 $\pm$ 1.29 & 19.41 $\pm$ 1.49 & 37.76 $\pm$ 4.07 & \textbf{51.89 $\pm$ 1.43} \\
ORL & 69.50 $\pm$ 0.55 & 60.22 $\pm$ 0.74 & 39.22 $\pm$ 0.00 & 47.60 $\pm$ 0.00 & 22.66 $\pm$ 0.00 & 74.10 $\pm$ 0.00 & 21.49 $\pm$ 0.00 & 15.32 $\pm$ 4.98 & 35.55 $\pm$ 8.75 & \textbf{78.05 $\pm$ 2.67} \\
UMIST & 66.91 $\pm$ 0.00 & 34.81 $\pm$ 0.47 & 32.12 $\pm$ 0.00 & 25.19 $\pm$ 0.00 & 26.63 $\pm$ 0.00 & 46.89 $\pm$ 0.00 & 22.99 $\pm$ 0.00 & 1.07 $\pm$ 0.31 & \textbf{79.36 $\pm$ 0.65} & 75.55 $\pm$ 2.69 \\
EYaleB & 19.67 $\pm$ 0.06 & 53.91 $\pm$ 5.11 & 1.87 $\pm$ 0.00 & -0.40 $\pm$ 0.00 & 13.65 $\pm$ 0.00 & \textbf{92.19 $\pm$ 0.00} & 21.67 $\pm$ 0.00 & 2.16 $\pm$ 0.31 & 64.15 $\pm$ 1.28 & 91.76 $\pm$ 0.90 \\
COIL-100 & 61.38 $\pm$ 1.84 & 21.98 $\pm$ 0.13 & 43.00 $\pm$ 0.00 & 43.93 $\pm$ 0.00 & 35.05 $\pm$ 0.00 & 54.03 $\pm$ 0.00 & 7.84 $\pm$ 0.00 & 1.59 $\pm$ 0.04 & \textbf{80.63 $\pm$ 1.28} & 66.60 $\pm$ 0.66 \\
\bottomrule
\end{tabular}}
\end{table*}

LoomSC achieves higher ACC than every scalable non-deep baseline on all five datasets. The improvements over the strongest method in this group range from $3.98$ to $8.17$ points. The gains over sampling-based methods are larger. Relative to the best ACC among LMVSC, SGL, and S$^5$C, LoomSC improves by $23.92$ points on ORL and $25.14$ points on COIL-100. The corresponding gains are $15.00$ points on UMIST and $72.32$ points on EYaleB. These comparisons show that a compact model need not be restricted to the clustering quality of fixed-representation sampling methods. LoomSC retains compactness while jointly refining the latent features and their self-expression factors.

On Fashion-MNIST, LoomSC obtains the highest ACC of $65.33\pm1.79\%$ and ARI of $51.89\pm1.43\%$. These exceed the strongest competing results by $3.98$ and $5.36$ points, respectively. Its NMI of $63.91\pm0.82\%$ is within $0.46$ points of the highest value. The agreement between ACC and ARI supports the quality of the partition beyond a single evaluation criterion.

On ORL, LoomSC leads all three metrics with $86.17\pm2.01\%$ ACC, $91.76\pm0.88\%$ NMI, and $78.05\pm2.67\%$ ARI. It improves upon LSR, the strongest alternative for each metric, by $3.67$, $1.69$, and $3.95$ points. These results are obtained with only ten images per subject. On EYaleB, LoomSC reaches $94.73\%$ ACC, $95.47\%$ NMI, and $91.76\%$ ARI. Its NMI closely matches LSR at $95.44\%$, and their ARI scores differ by only $0.43$ points. Thus, the thin-factor representation can match or improve upon dense least-squares self-expression on several clustering measures without storing all pairwise coefficients.

The comparison with PRO-DSC assesses whether this efficiency is compatible with deep clustering quality. LoomSC exceeds PRO-DSC on all three metrics for Fashion-MNIST, ORL, and EYaleB. The ACC gains are $13.21$, $34.17$, and $17.10$ points, respectively. Across the five datasets, its average improvements over PRO-DSC are $11.08$ ACC points, $6.72$ NMI points, and $13.28$ ARI points. Therefore, the advantage over the non-deep baselines is accompanied by a strong comparison with a deep full-affinity method.

PRO-DSC attains the highest scores on UMIST and COIL-100, where LoomSC ranks second on all three metrics. On UMIST, LoomSC reaches $80.00\%$ ACC compared with $81.39\%$ for PRO-DSC. On COIL-100, it obtains $74.64\%$ ACC, $90.95\%$ NMI, and $66.60\%$ ARI. Its COIL-100 ACC exceeds every non-deep baseline, with gains of $7.14$ points over EnSC and $13.54$ points over LSR. These results establish competitive deep clustering performance on both datasets. As quantified next, LoomSC attains this performance at a small fraction of PRO-DSC's measured runtime and without its non-linear-time full-affinity spectral stage.

\subsection{Computational efficiency}

Table~\ref{tab:time} reports the measured wall-clock times. LoomSC requires $0.239$~s on ORL and $0.203$~s on UMIST. Its mean time is $2.240$~s on EYaleB, $9.008$~s on Fashion-MNIST, and $25.311$~s on COIL-100. No competing method combines higher mean ACC with lower mean runtime on any of these datasets. Thus, LoomSC lies on the empirical ACC--runtime Pareto frontier for every benchmark among the evaluated implementations.

\begin{table*}[t]
\centering
\setlength{\tabcolsep}{3pt}
\caption{Measured wall-clock time (s). Lower is better. Bold indicates the smallest mean in each row.}
\label{tab:time}
\resizebox{\textwidth}{!}{%
\begin{tabular}{lcccccccccccc}
\toprule
\multicolumn{2}{c}{} & \multicolumn{8}{c}{Non-deep methods} & \multicolumn{2}{c}{Deep methods} \\
\cmidrule(lr){3-10}\cmidrule(lr){11-12}
Dataset & $n$ & EnSC & SSC-OMP & LMVSC & SGL & S$^5$C & LSR & A-DSSC & SSSC & PRO-DSC & LoomSC \\
\midrule
Fashion-MNIST & 5,000 & 3219.752 $\pm$ 46.948 & \textbf{5.447 $\pm$ 0.681} & 16.965 $\pm$ 0.179 & 17.203 $\pm$ 0.219 & 8.280 $\pm$ 1.000 & 6.260 $\pm$ 0.157 & 35.612 $\pm$ 0.584 & 2231.314 $\pm$ 69.877 & 563.338 $\pm$ 1.392 & 9.008 $\pm$ 0.987 \\
ORL & 400 & 7.953 $\pm$ 0.514 & 1.856 $\pm$ 0.058 & 0.929 $\pm$ 0.031 & 0.958 $\pm$ 0.003 & 7.440 $\pm$ 0.419 & \textbf{0.049 $\pm$ 0.001} & 0.699 $\pm$ 0.005 & 7.084 $\pm$ 0.183 & 102.797 $\pm$ 2.671 & 0.239 $\pm$ 0.089 \\
UMIST & 575 & 24.366 $\pm$ 26.899 & 0.498 $\pm$ 0.012 & 1.086 $\pm$ 0.033 & 1.263 $\pm$ 0.043 & 10.249 $\pm$ 4.946 & \textbf{0.077 $\pm$ 0.002} & 1.026 $\pm$ 0.027 & 9.774 $\pm$ 1.423 & 78.689 $\pm$ 0.172 & 0.203 $\pm$ 0.003 \\
EYaleB & 2,432 & 43.843 $\pm$ 0.398 & 12.666 $\pm$ 0.975 & 9.124 $\pm$ 0.063 & 9.221 $\pm$ 0.083 & 471.768 $\pm$ 24.647 & 2.389 $\pm$ 0.542 & 13.918 $\pm$ 0.132 & 2286.421 $\pm$ 94.488 & 313.648 $\pm$ 2.670 & \textbf{2.240 $\pm$ 0.581} \\
COIL-100 & 7,200 & 50.858 $\pm$ 1.523 & 1319.586 $\pm$ 164.863 & \textbf{21.421 $\pm$ 0.143} & 22.498 $\pm$ 0.390 & 4709.568 $\pm$ 169.023 & 49.888 $\pm$ 5.451 & 175.328 $\pm$ 1.551 & 2409.599 $\pm$ 245.492 & 7443.627 $\pm$ 4.992  & 25.311 $\pm$ 9.976 \\
\bottomrule
\end{tabular}}
\end{table*}

The distinction from PRO-DSC concerns the complete clustering pipeline rather than network training alone. PRO-DSC reduces training costs through mini-batches, but its published inference constructs a full pairwise affinity before spectral clustering~\cite{loom_pro_dsc}. Explicitly storing this affinity requires $O(n^2)$ memory for a clustering set of size $n$. LoomSC instead retains $P\in\mathbb R^{n\times m}$ and $L\in\mathbb R^{d_z\times m}$ during representation learning. Its quadratic feature map transfers spectral computation to dimension $q=m(m+1)/2$, which is independent of $n$ for fixed $m$. Consequently, it removes the full-affinity bottleneck rather than only accelerating its construction. Section~\ref{sec:loom-complexity} establishes linear time and memory in $n$ under fixed dimensions and iteration budgets.

The reported runtimes show substantial savings relative to the deep baseline. LoomSC is $62.5\times$ faster than PRO-DSC on Fashion-MNIST while improving ACC by $13.21$ points. On ORL, it is $430.1\times$ faster and improves ACC by $34.17$ points. On EYaleB, the corresponding speedup is $140.0\times$, together with a $17.10$-point ACC gain. Hence, on three datasets, the computational reduction accompanies higher clustering quality rather than a loss of accuracy.

The datasets where PRO-DSC achieves its strongest results further illustrate the scalability advantage. On UMIST, LoomSC attains $80.00\%$ ACC in $0.203$~s, compared with $81.39\%$ in $78.689$~s for PRO-DSC. This is a $387.6\times$ speedup with an ACC difference of $1.39$ points. On COIL-100, LoomSC obtains the second-highest ACC while reducing the measured time from $7{,}443.627$ to $25.311$~s, a $294.1\times$ speedup. Across all five datasets, the speedups over PRO-DSC range from $62.5\times$ to $430.1\times$. Therefore, competitive deep clustering performance is retained without the runtime of the full-affinity reference.

LoomSC also improves the quality--cost balance relative to scalable non-deep methods. On Fashion-MNIST, it exceeds EnSC by $3.98$ ACC points with a $357.4\times$ speedup. On COIL-100, it improves ACC over EnSC and LSR by $7.14$ and $13.54$ points while approximately halving their runtimes. Compared with LMVSC, the fastest method on COIL-100, LoomSC requires only $3.89$ additional seconds and improves ACC by $25.14$ points. These results show that LoomSC combines runtimes close to or below those of sampling methods with substantially stronger clustering. Its compact pipeline also remains far faster than the deep full-affinity reference.

\begin{figure*}
\centering
\includegraphics[width=0.68\textwidth,viewport=460 304.536 963 353.536,clip]{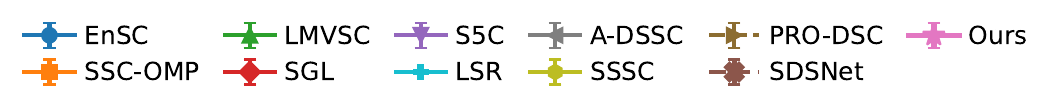}\\[3pt]
\begin{minipage}{0.48\textwidth}
\centering
\includegraphics[width=\linewidth,viewport=0 3.536 349 303.536,clip]{images/figure3_performance.pdf}
\end{minipage}\hfill
\begin{minipage}{0.48\textwidth}
\centering
\includegraphics[width=\linewidth,viewport=355 3.536 703 303.536,clip]{images/figure3_performance.pdf}
\end{minipage}\\[4pt]
\begin{minipage}{0.48\textwidth}
\centering
\includegraphics[width=\linewidth,viewport=705 3.536 1063 303.536,clip]{images/figure3_performance.pdf}
\end{minipage}\hfill
\begin{minipage}{0.48\textwidth}
\centering
\includegraphics[width=\linewidth,viewport=1070 3.536 1424.709 303.536,clip]{images/figure3_performance.pdf}
\end{minipage}
\caption{Synthetic-data comparison with ten subspaces and $n_i$ samples per subspace. Panels show (a) ACC, (b) connectivity, (c) CPU runtime, and (d) GPU runtime. ``Ours'' denotes LoomSC.}
\label{fig:synthetic_performance}
\end{figure*}

\subsection{Scalability Analysis}

We first vary the sampling density in a controlled union-of-subspaces setting. Following the synthetic-data configuration in~\cite{paper18}, we generate $k=10$ subspaces with random dimensions between $6$ and $12$ in an ambient space of dimension $d_x=784$. Each subspace contains $n_i\in\{30,50,100,200,500,1000\}$ samples. The total sample count is therefore $n=10n_i$, ranging from $300$ to $10{,}000$. Figure~\ref{fig:synthetic_performance} compares ACC, connectivity, and CPU and GPU runtimes.

LoomSC maintains $99\%$--$100\%$ ACC across the entire density range. Its connectivity remains high, at $0.68$--$0.75$. Several competing methods attain high ACC at sufficiently large sample counts but have much smaller connectivity scores. Thus, the two diagnostics provide complementary evidence: LoomSC produces accurate partitions while retaining strong graph connectivity under the reported measure. This behavior is consistent with the goal of preserving informative sample relations through the factorized representation.

The comparison with the SDSNet~\cite{paper38} reference is particularly informative at low sampling densities. With $30$ samples per subspace, LoomSC reaches $99\%$ ACC, compared with about $85\%$ for SDSNet. With $50$ samples per subspace, its ACC is $100\%$, compared with about $87\%$. At the largest size, the two approaches both retain near-perfect accuracy and high connectivity. LoomSC nevertheless reduces the measured runtime from $143$ to $39$~s in both the CPU and GPU panels. The improvement is visible within each hardware setting rather than arising only from a CPU-to-GPU comparison.

The same experiment contrasts LoomSC with the full-affinity deep baseline. PRO-DSC also reaches near-perfect ACC at higher sampling densities. However, its CPU runtime is substantially larger at the highest density. LoomSC combines high clustering accuracy across the full density range with lower computational cost at the largest size in this experiment.

Figure~\ref{fig:synthetic_scalability} extends the sample-size experiment from $1{,}000$ to $500{,}000$ points. LoomSC maintains $99.8\%$ ACC or higher throughout this $500$-fold increase. Its runtime grows from about $57$~s at $50{,}000$ samples to $362$~s at $500{,}000$ samples. At the largest size, it retains near-perfect accuracy in approximately six minutes. S$^5$C reaches comparable accuracy but requires about $1{,}054$~s. LMVSC has a lower runtime of about $125$~s, with ACC near $94\%$. SGL also reaches near-perfect ACC at large sizes, whereas LoomSC maintains this accuracy regime across the full tested range, including the smallest datasets.

\begin{figure*}[t]
\centering
\includegraphics[width=\textwidth]{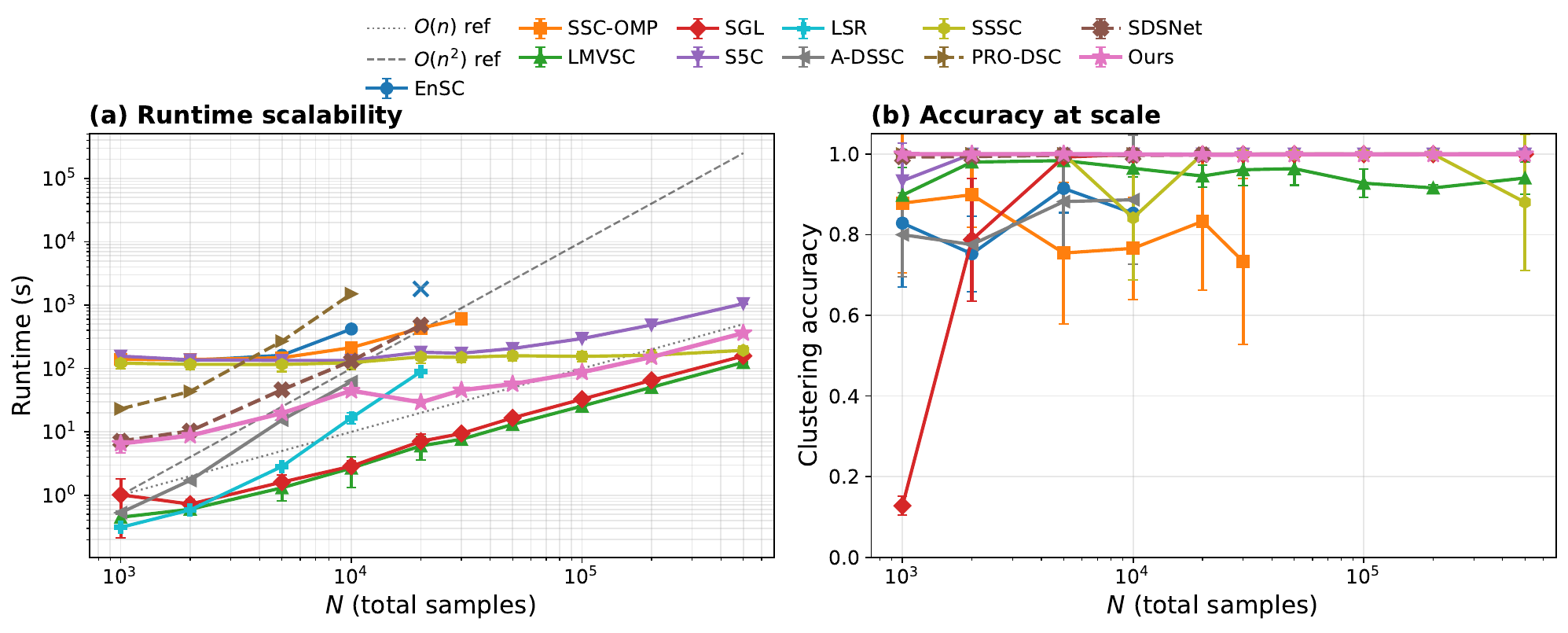}
\caption{Synthetic-data scalability from $1{,}000$ to $500{,}000$ samples: (a) runtime and (b) clustering accuracy. The horizontal-axis symbol $N$ denotes the total sample count $n$. The reference curves indicate linear and quadratic growth. ``Ours'' denotes LoomSC.}
\label{fig:synthetic_scalability}
\end{figure*}

The large-sample runtime trend is consistent with the linear-in-$n$ bounds in Section~\ref{sec:loom-complexity}. For fixed factor width, the quadratic feature dimension $q=m(m+1)/2$ does not grow with the number of samples. Neither factor learning nor the reduced spectral stage requires storage of all pairwise affinities. With fixed dimensions and iteration budgets, the sample-dependent work remains linear. The experiment shows that this design supports increasing the dataset size without a corresponding loss of clustering accuracy.

\subsection{Convergence of the Affinity Matrix}

Figure~\ref{fig:affinity_convergence} examines the relative Frobenius change between successive affinity iterates. The diagnostic divides the Frobenius norm of the difference by the norm of the previous iterate. This difference measures the stability of the sample relationships used for clustering, rather than the scale of the matrix alone.

\begin{figure*}[t]
\centering
\includegraphics[width=\textwidth]{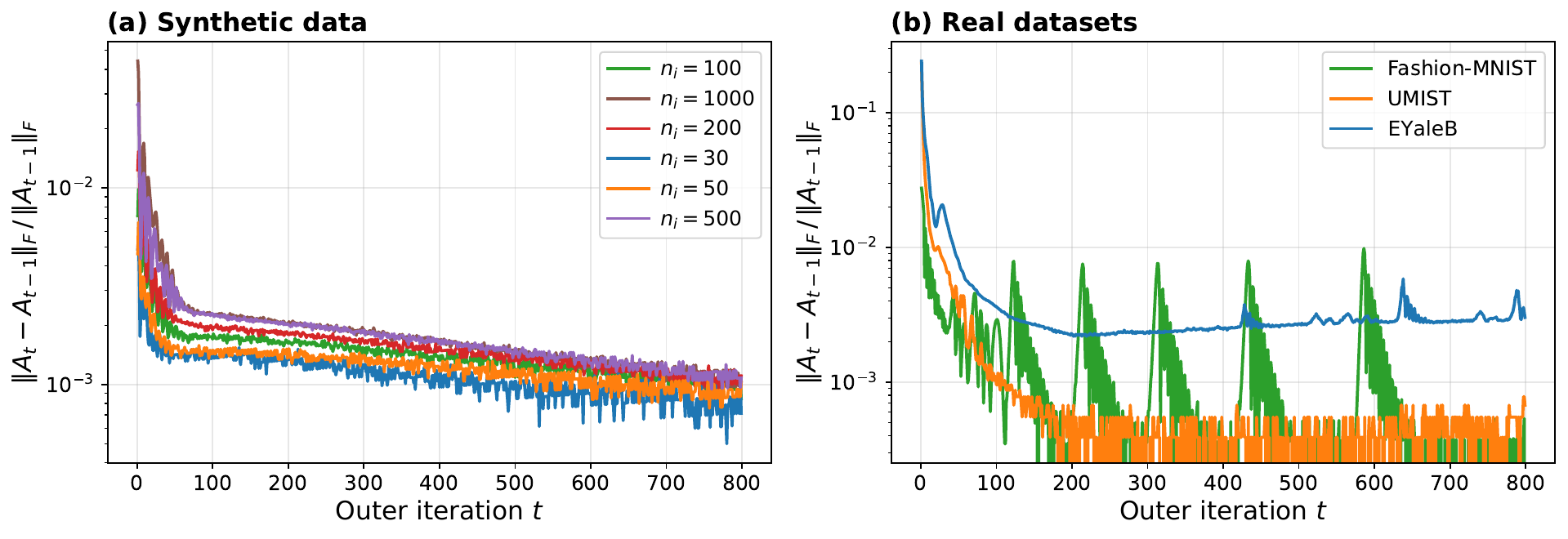}
\caption{Relative Frobenius change $\|A_t-A_{t-1}\|_F/\|A_{t-1}\|_F$ on (a) synthetic datasets with ten subspaces and (b) real datasets. Here, $A_t$ denotes the affinity matrix in~\eqref{eq:loom-affinity} at outer iteration $t$, and $n_i$ is the number of samples in synthetic subspace $i$.}
\label{fig:affinity_convergence}
\end{figure*}

On synthetic data, all six sampling densities exhibit a short initial transient followed by gradual stabilization. The relative change falls to a few $10^{-3}$ within roughly the first $100$ outer iterations and approaches $10^{-3}$ by the end of the run. The curves reach similar terminal scales despite the increase from $30$ to $1{,}000$ samples per subspace. Thus, the larger synthetic datasets retain the same qualitative stabilization pattern.

The real-data curves show similarly small late-stage changes. UMIST falls below $10^{-3}$ after about $100$ iterations. Fashion-MNIST reaches a low baseline with temporary excursions, while EYaleB settles around a few $10^{-3}$. The dominant pattern is a rapid reduction in the update magnitude followed by sustained small changes. This supports the practical stability of the learned affinities during joint training. Proposition~\ref{prop:loom-descent} provides a complementary descent result for the exact factor updates with fixed latent features. The observed curves show that small affinity changes are also attained when these updates are coupled with representation learning.

\begin{figure*}
\centering
\begin{minipage}[b]{0.32\textwidth}
  \centering
  \includegraphics[width=\linewidth]{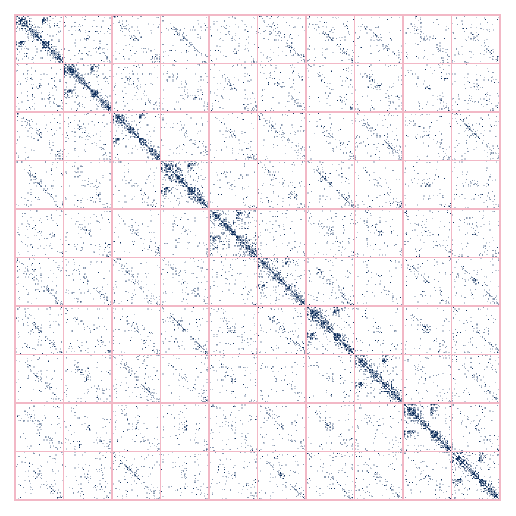}
  \\[2pt] \small (a) Epoch 10
\end{minipage}
\hfill
\begin{minipage}[b]{0.32\textwidth}
  \centering
  \includegraphics[width=\linewidth]{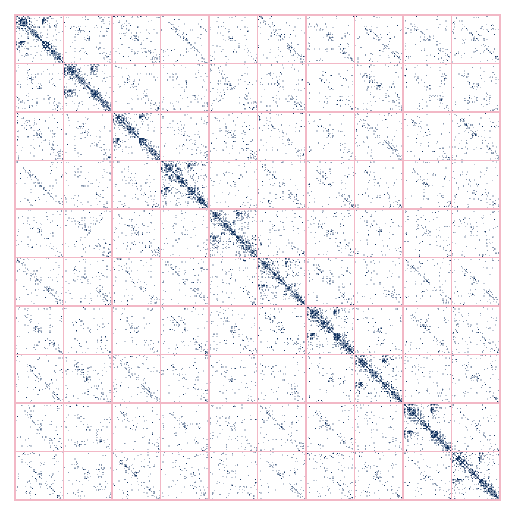}
  \\[2pt] \small (b) Epoch 100
\end{minipage}
\hfill
\begin{minipage}[b]{0.32\textwidth}
  \centering
  \includegraphics[width=\linewidth]{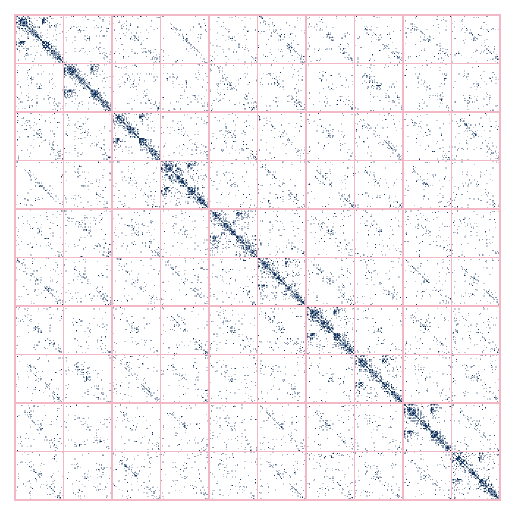}
  \\[2pt] \small (c) Epoch 500
\end{minipage}
\caption{Learned sample relations on a ten-class EYaleB subset at epochs $10$, $100$, and $500$. Class boundaries are overlaid.}
\label{fig:aff_yaleb}
\end{figure*}

\begin{figure*}
\centering
\begin{minipage}[b]{0.32\textwidth}
  \centering
  \includegraphics[width=\linewidth]{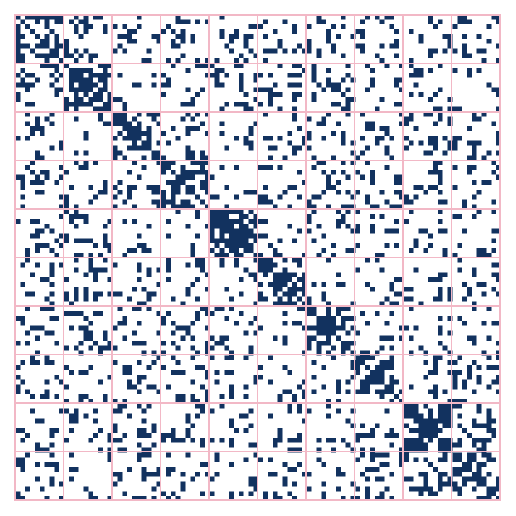}
  \\[2pt] \small (a) Epoch 10
\end{minipage}
\hfill
\begin{minipage}[b]{0.32\textwidth}
  \centering
  \includegraphics[width=\linewidth]{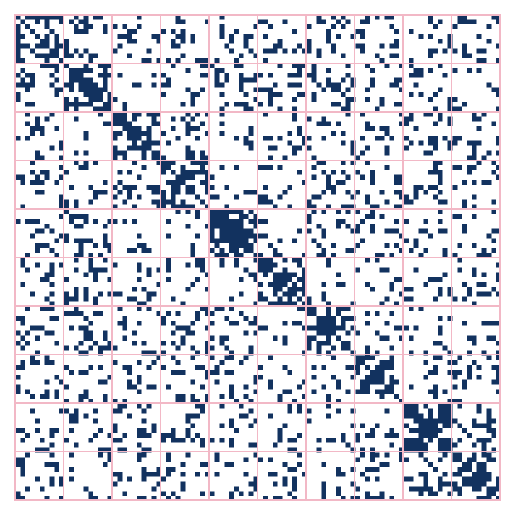}
  \\[2pt] \small (b) Epoch 100
\end{minipage}
\hfill
\begin{minipage}[b]{0.32\textwidth}
  \centering
  \includegraphics[width=\linewidth]{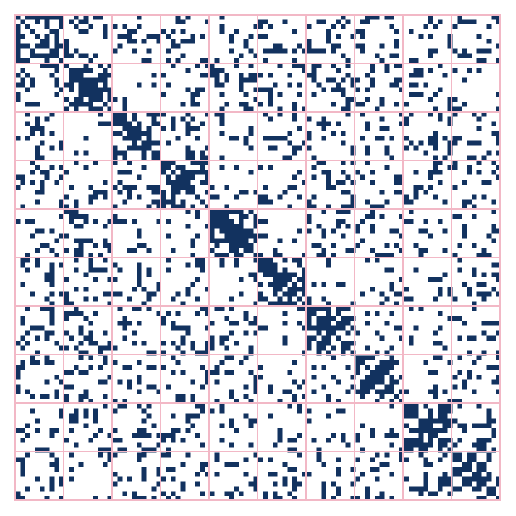}
  \\[2pt] \small (c) Epoch 500
\end{minipage}
\caption{Learned sample relations on a ten-class ORL subset at epochs $10$, $100$, and $500$. Class boundaries are overlaid.}
\label{fig:aff_orl}
\end{figure*}

\begin{figure*}
\centering
\begin{minipage}[b]{0.32\textwidth}
  \centering
  \includegraphics[width=\linewidth]{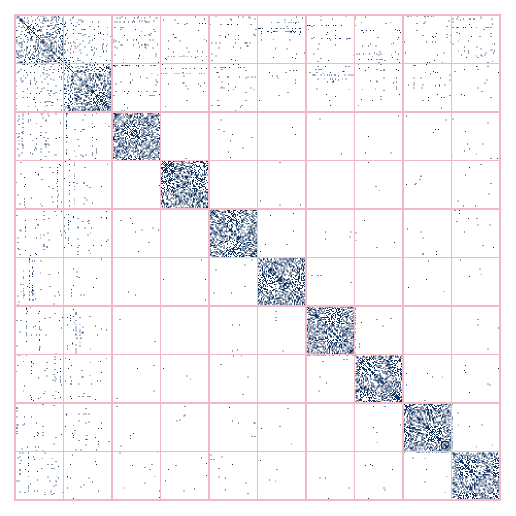}
  \\[2pt] \small (a) Epoch 10
\end{minipage}
\hfill
\begin{minipage}[b]{0.32\textwidth}
  \centering
  \includegraphics[width=\linewidth]{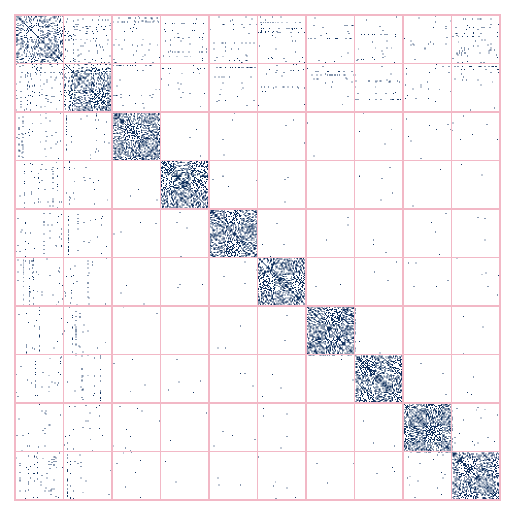}
  \\[2pt] \small (b) Epoch 100
\end{minipage}
\hfill
\begin{minipage}[b]{0.32\textwidth}
  \centering
  \includegraphics[width=\linewidth]{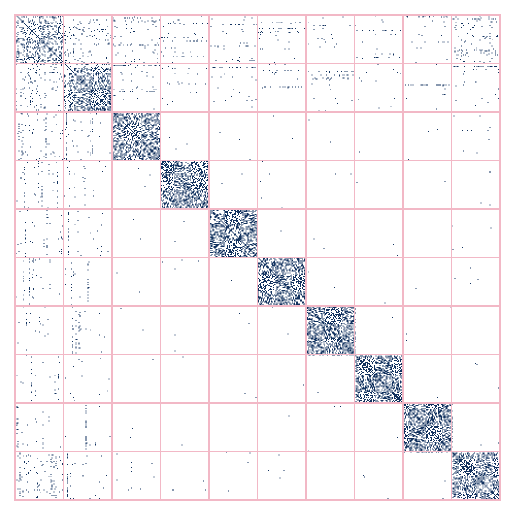}
  \\[2pt] \small (c) Epoch 500
\end{minipage}
\caption{Learned sample relations on synthetic data with ten subspaces at epochs $10$, $100$, and $500$. Subspace boundaries are overlaid.}
\label{fig:aff_syn_epoch}
\end{figure*}

\begin{figure*}
\centering
\begin{minipage}[b]{0.32\textwidth}
  \centering
  \includegraphics[width=\linewidth]{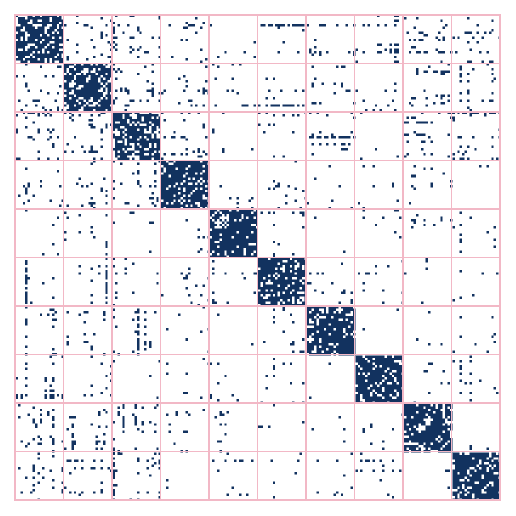}
  \\[2pt] \small (a) $n = 200$
\end{minipage}
\hfill
\begin{minipage}[b]{0.32\textwidth}
  \centering
  \includegraphics[width=\linewidth]{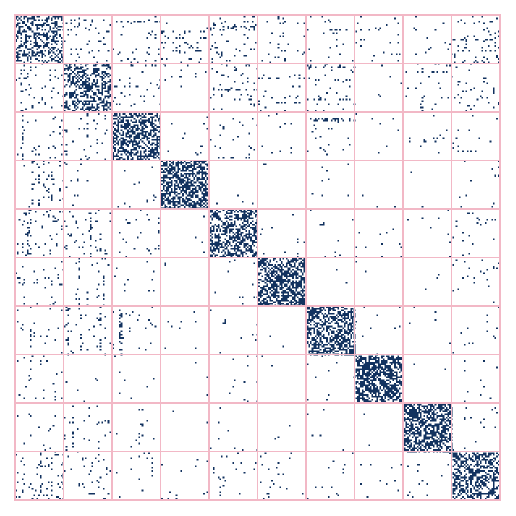}
  \\[2pt] \small (b) $n = 300$
\end{minipage}
\hfill
\begin{minipage}[b]{0.32\textwidth}
  \centering
  \includegraphics[width=\linewidth]{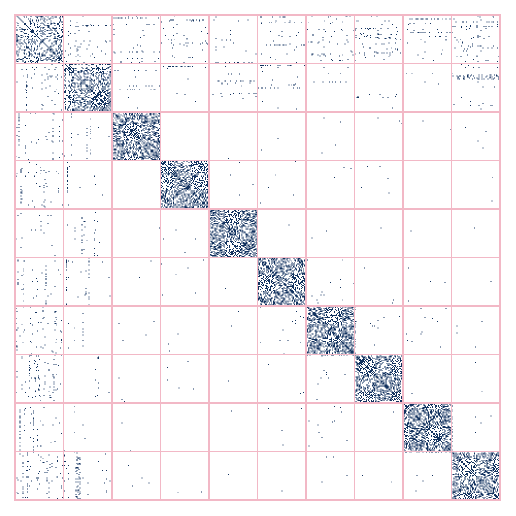}
  \\[2pt] \small (c) $n = 500$
\end{minipage}
\caption{Learned sample relations on synthetic data with ten subspaces and increasing total sample count at epoch $500$. Subspace boundaries are overlaid.}
\label{fig:aff_syn_size}
\end{figure*}

\subsection{Affinity-matrix structure}

Figures~\ref{fig:aff_yaleb}--\ref{fig:aff_syn_epoch} visualize the learned sample relations at epochs $10$, $100$, and $500$ after pretraining. The face-data panels use ten-class subsets, and the synthetic panels contain ten subspaces. Samples are grouped by their classes, with class boundaries overlaid. These plots examine whether the representation retains class-aligned relations throughout optimization.

On EYaleB, the diagonal blocks are already visible after 10 epochs and remain clearly identifiable at later epochs. The within-class regions contain more pronounced structure than most off-block regions. ORL also retains recognizable diagonal blocks despite having only ten images per subject. This visual organization is consistent with the strong face-clustering scores in Tables~\ref{tab:acc}--\ref{tab:ari}. In particular, the displayed relations extend beyond individual self-connections and include associations between distinct samples from the same class.

The synthetic panels exhibit an even clearer separation between dense diagonal blocks and the surrounding regions. This organization persists across the training stages (i.e., $10$, $100$, and $500$ epochs). Figure~\ref{fig:aff_syn_size} compares $n=200$, $300$, and $500$ at epoch $500$. The 10 diagonal blocks remain identifiable as more samples are added, and the larger panels show a clearer visual contrast between within-subspace structure and off-block regions. Together with the quantitative connectivity results, these visualizations support the ability of the affinity matrix to retain structured sample relations as the dataset grows.

\section{Conclusion}

This work presents LoomSC, a scalable subspace clustering framework that combines projector-based self-expression with exact spectral reduction. Jointly learning latent representations and two thin factors avoids explicit sample-by-sample coefficient matrices. A nonnegative quadratic affinity preserves the projector's support and enables normalized spectral clustering without forming the full affinity or sample Laplacian. For fixed dimensions and iteration budgets, the complete pipeline has linear time and memory complexity in the number of samples. Our analysis quantifies the projector approximation to regularized least squares. It also identifies conditions under which the full row-space projector preserves subspace memberships and ensures within-subspace connectivity.

Across five image-clustering benchmarks, LoomSC achieves the highest average ACC, NMI, and ARI among the evaluated methods. Synthetic experiments extend to $500{,}000$ samples while maintaining at least $99.8\%$ accuracy throughout the tested range. These findings demonstrate that compact self-expression can combine strong clustering quality with scalability across both representation learning and spectral assignment.

\bibliographystyle{IEEEtran}
\bibliography{references}

\clearpage
\appendices
\section{Spectral Analysis of the LSR Approximation}
\label{app:loom-lsr}

The normal equation of~\eqref{eq:loom-ridge} is
\begin{equation}
 (X^\top X+2\lambda I_n)C=X^\top X.
 \label{eq:loom-lsr-normal}
\end{equation}
The coefficient matrix is positive definite for $\lambda>0$, so the solution is unique. Since $X^\top X=V_x\operatorname{diag}(\beta_1,\ldots,\beta_r)V_x^\top$, the vectors in $V_x$ diagonalize the solution on the row space of $X$. Along its $j$th eigenvector direction, the normal equation becomes
$(\beta_j+2\lambda)\tau_j=\beta_j$.
In the orthogonal complement, the solution has an eigenvalue of zero. This proves Proposition~\ref{prop:loom-lsr-spectrum}.

The LSR solution and $C_{\mathrm{row}}$ have the same eigenvectors on the row space. The invariance of the Frobenius norm under orthonormal changes of coordinates therefore gives
\begin{equation}
 \|C^\star_{\mathrm{LSR}}-C_{\mathrm{row}}\|_F^2
 =\sum_{j=1}^{r}(1-\tau_j)^2,
 \label{eq:loom-lsr-error-proof}
\end{equation}
which proves~\eqref{eq:loom-ridge-gap}. For the rank-$m$ projector $V_{x,m}V_{x,m}^\top$, the first $m$ eigenvalues are one and the remaining eigenvalues are zero. Splitting the squared difference over these two sets proves~\eqref{eq:loom-truncated-approximation}. These identities quantify the spectral approximation used to motivate the factor model.

\section{Proofs for the Factor Model}
\label{app:loom-factor}

For fixed $Z\in\mathbb R^{d_z\times n}$ and feasible $P\in\mathbb R^{n\times m}$, decompose the residual as
\begin{equation}
 Z-LP^\top
 =Z(I_n-PP^\top)+(ZP-L)P^\top.
 \label{eq:loom-app-split}
\end{equation}
The two terms are orthogonal in the Frobenius inner product because $(I_n-PP^\top)P=0$. Moreover, $\|(ZP-L)P^\top\|_F^2=\|ZP-L\|_F^2$ since $P^\top P=I_m$. Taking squared norms proves~\eqref{eq:loom-split}. The second term in that identity is strictly convex in $L$. Thus, $L=ZP$ is its unique minimizer. The equality $LP^\top=ZPP^\top$ holds after this update.

For fixed $L$, expansion gives~\eqref{eq:loom-procrustes-objective}. Let $Z^\top L=U_P\Sigma_PV_P^\top$ have the dimensions specified in \eqref{eq:loom-p-update}. The trace inequality yields
\begin{equation}
 \operatorname{tr}(P^\top Z^\top L)
 \leq\operatorname{tr}(\Sigma_P).
 \label{eq:loom-app-trace}
\end{equation}
The feasible matrix $P=U_PV_P^\top$ attains equality. This proves the global optimality of the Procrustes update. When $Z^\top L$ is rank deficient, zero singular directions must be included so that both singular bases have $m$ columns. Different orthonormal completions can yield different minimizers. 

To prove Proposition~\ref{prop:loom-factor}, eliminate $L$ and expand the remaining objective:
\begin{equation}
 \|Z-ZPP^\top\|_F^2
 =\|Z\|_F^2-\operatorname{tr}(P^\top Z^\top ZP).
 \label{eq:loom-app-rayleigh}
\end{equation}
The variational characterization of eigenvalues maximizes the trace over an $m$-dimensional dominant eigenspace of $Z^\top Z$. Its eigenvalues are $\sigma_1^2,\ldots,\sigma_{r_z}^2$, followed by zeros. Subtracting their largest $m$ values from $\|Z\|_F^2$ gives \eqref{eq:loom-tail}. When $m<r_z$, a spectral tie at the truncation boundary can make the optimal subspace nonunique.

When $m=r_z$, the minimum residual is zero. Hence, every global minimizer satisfies $ZPP^\top=Z$, which implies $\operatorname{range}(Z^\top)\subseteq\operatorname{range}(P)$.
Both spaces have dimension $r_z$, so they coincide. Consequently, $C=PP^\top=V_zV_z^\top=Z^\dagger Z$ is the unique row-space projector, where $\dagger$ denotes the Moore--Penrose pseudoinverse. Its orthonormal basis $P$ remains nonunique. If $m>r_z$, the same inclusion holds, but $\operatorname{range}(P)$ also contains $m-r_z$ directions orthogonal to the row space. This proves the remaining assertions of Proposition~\ref{prop:loom-factor}.

For Proposition~\ref{prop:loom-descent}, each exact block update minimizes the same fixed-$Z$ objective with its other block held constant. Neither update can increase that objective. Its successive values are nonnegative and nonincreasing, so they converge.

\section{Subspace Preservation and Connectivity}
\label{app:loom-recovery}

Reorder the samples by subspace and write $Z=[Z_1,\ldots,Z_k]$, where $Z_a\in\mathbb R^{d_z\times n_a}$, $\sum_{a=1}^k n_a=n$, and $\operatorname{span}(Z_a)=S_a$. Under Proposition~\ref{prop:loom-recovery}, $S_a$ has dimension $r_a>0$ and the subspaces form a direct sum. Choose a basis matrix $\Omega_a\in\mathbb R^{d_z\times r_a}$ and write $Z_a=\Omega_a\Xi_a$, where $\Xi_a\in\mathbb R^{r_a\times n_a}$ has full row rank. Define
\begin{equation}
 \begin{aligned}
 \Omega&=[\Omega_1,\ldots,\Omega_k]
       \in\mathbb R^{d_z\times r_z},\\
 \Xi&=\operatorname{blkdiag}(\Xi_1,\ldots,\Xi_k)
       \in\mathbb R^{r_z\times n}.
 \end{aligned}
 \label{eq:loom-app-coordinates}
\end{equation}
The direct-sum condition makes $\Omega$ full column rank. Therefore, $Z=\Omega\Xi$ and $\Xi=\Omega^\dagger Z$, which imply $\operatorname{row}(Z)=\operatorname{row}(\Xi)$. Their orthogonal row-space projectors coincide:
\begin{equation}
 \begin{aligned}
 Z^\dagger Z
 &=\Xi^\dagger\Xi\\
 &=\operatorname{blkdiag}
   (Z_1^\dagger Z_1,\ldots,Z_k^\dagger Z_k).
 \end{aligned}
 \label{eq:loom-app-projector-blocks}
\end{equation}
The second equality uses the block structure of $\Xi$ and $\operatorname{row}(\Xi_a)=\operatorname{row}(Z_a)$. At a global fixed-$Z$ optimum with $m=r_z$, Appendix~\ref{app:loom-factor} gives $C=Z^\dagger Z$. Equation~\eqref{eq:loom-app-projector-blocks} therefore proves subspace preservation.

For connectivity, the stated sample condition can be written
\begin{equation}
 \operatorname{rank}(Z_{\mathcal A})
 +\operatorname{rank}(Z_{\mathcal B})>r_a
 \label{eq:loom-app-indecomposable}
\end{equation}
for every partition of a class's indices into two nonempty sets $\mathcal A$ and $\mathcal B$. Here $Z_{\mathcal A}$ and $Z_{\mathcal B}$ contain the corresponding columns of $Z_a$. Indeed, their spans sum to $S_a$. Equality of their rank sum with $r_a$ is equivalent to those spans forming a direct sum.

Set $C_a=Z_a^\dagger Z_a$. Suppose the graph defined by its nonzero off-diagonal entries is disconnected. There is a nontrivial partition $\mathcal A,\mathcal B$ for which $C_a$ is block diagonal. Since $\operatorname{row}(C_a)=\operatorname{row}(Z_a)$, this row space then splits into two spaces supported on the respective coordinate sets. Coordinate projection onto $\mathcal A$ has dimension $\operatorname{rank}(Z_{\mathcal A})$, which also equals the rank of the corresponding diagonal block of $C_a$. The analogous equality holds for $\mathcal B$. Adding the two block ranks gives $\operatorname{rank}(Z_{\mathcal A})+ \operatorname{rank}(Z_{\mathcal B})=r_a$, contradicting~\eqref{eq:loom-app-indecomposable}. Conversely, equality in this rank sum gives a direct-sum partition. Applying~\eqref{eq:loom-app-projector-blocks} to its two groups makes $C_a$ block diagonal. Thus, the sample condition is also necessary
for connectivity of the full row-space projector graph. A singleton containing a nonzero sample has a connected one-vertex graph and satisfies the condition vacuously.

Every diagonal entry of $C_a$ is positive when all samples are nonzero. A zero diagonal in an orthogonal projector forces the corresponding row and column to vanish. The identity $Z_a=Z_aC_a$ would then make the corresponding sample zero, a contradiction. Positive row rescaling and entrywise squaring preserve the off-diagonal support. The affinity specified in the method therefore has exactly $k$ connected components under these assumptions.

The eigenspace of eigenvalue one of its normalized affinity is spanned by the degree-weighted component indicators. Any orthonormal basis of this eigenspace differs from the normalized indicator basis by an orthogonal rotation. Row normalization maps every sample of one component to the same direction. Distinct components yield distinct directions. Consequently, the ideal embedding admits zero-distortion $k$-means assignments equal to the components, up to a label permutation. 

\section{Properties of the Quadratic Affinity and Spectral Reduction}
\label{app:loom-spectral-proofs}

\paragraph{Affinity and degrees}
The definition of $\operatorname{svec}$ gives $\operatorname{svec}(R)^\top\operatorname{svec}(S) =\operatorname{tr}(R^\top S)$ for symmetric $R,S\in\mathbb{R}^{m\times m}$. Applying this identity to the rank-one matrices $ p _i p _i^\top$ and $ p _j p _j^\top$ proves \eqref{eq:loom-affinity-factor}. Hence $A$ is positive semidefinite. Its entrywise nonnegativity follows from \eqref{eq:loom-affinity}. Let $e_i\in\mathbb R^n$ be the $i$th standard basis vector. Since $C^2=C$,
\begin{equation}
 \begin{aligned}
 \delta_i
 &= e _i^\top CC^\top e _i
   = e _i^\top C e _i
   =\| p _i\|_2^2,\\
 \sum_i\delta_i&=\operatorname{tr}(C)=m.
 \end{aligned}
 \label{eq:loom-degree-proof}
\end{equation}
The diagonal entries of an orthogonal projector belong to $[0,1]$, so $0\leq\delta_i\leq1$. Moreover, $\delta_i=0$ if and only if $ p _i=0$, which makes the entire corresponding affinity row zero. Removing these rows, therefore, leaves the degrees of all active samples unchanged.

\paragraph{Proof of Proposition~\ref{prop:loom-spectral}}
Let $s=\operatorname{rank}(F_+)$. Choose $s$ rows of $F_+$ spanning its row space. The corresponding $s$ rank-one matrices span all matrices $ p _i p _i^\top$. In particular, their span contains
\begin{equation}
 \sum_{i\in\mathcal I_+} p _i p _i^\top
 =P^\top P=I_m.
 \label{eq:loom-lift-rank-proof}
\end{equation}
A linear combination of $s$ rank-one matrices has rank at most $s$. Thus, $m\leq s$. The upper bound follows from the dimensions of $F_+$. Since $\Delta$ is positive definite, $B$ and $F_+$ have the same rank.

If $G v =\nu v $ for $\nu>0$, then
$(BB^\top)(B v )=B(G v )=\nu B v $
and $\|B v \|_2^2=\nu\| v \|_2^2$.
Conversely, a positive eigenpair of $BB^\top$ maps to an eigenpair of $G$ by multiplication by $B^\top$. For orthonormal eigenvectors,
\begin{equation}
 \frac{ v _a^\top B^\top B v _b}
      {\sqrt{\nu_a\nu_b}}
 = 1 \{a=b\}.
 \label{eq:loom-lifting-proof}
\end{equation}
This proves the correspondence, multiplicities, and orthonormality.

\paragraph{Leading vector and spectral convention}
Let $ \omega =\operatorname{svec}(I_m)$. We have
$F_+^\top 1 _{n_+}= \omega $,
$F_+ \omega =(\delta_i)_{i\in\mathcal I_+}$, and
$\| \omega \|_2^2=m$. Therefore,
\begin{equation}
 G \omega 
 =F_+^\top\Delta^{-1}F_+ \omega 
 =F_+^\top 1 _{n_+}
 = \omega .
 \label{eq:loom-leading-proof}
\end{equation}
This proves~\eqref{eq:loom-leading-vector}. Let $A_{++}\in\mathbb R^{n_+\times n_+}$ be the principal submatrix of $A$ indexed by $\mathcal I_+$. The normalized affinity $BB^\top$ is similar to the row-stochastic matrix $\Delta^{-1}A_{++}$. Its eigenvalues have a magnitude of at most one. Positive semidefiniteness further restricts them to $[0,1]$. Thus, the retained eigenvalue is dominant. Locking this vector remains valid when the eigenvalue one is repeated. Every active row then contains the strictly positive coordinate $\sqrt{\delta_i/m}$, which justifies the row normalization in~\eqref{eq:loom-spectral-embedding}.

\begin{proposition}[Conditional recovery on the active graph]
\label{prop:loom-component-recovery}
Suppose $A_{++}$ has exactly $k$ connected components. The row-normalized embedding in~\eqref{eq:loom-spectral-embedding}, computed from exact eigenvectors, has one distinct row per component. An exact solution of $k$-means therefore recovers those components up to label permutation.
\end{proposition}
\begin{proof}
Let $\Upsilon\in\{0,1\}^{n_+\times k}$ be the component membership matrix. The eigenvalue-one eigenspace of $BB^\top$ has dimension $k$ and an orthonormal basis
\begin{equation}
 \Delta^{1/2}\Upsilon(\Upsilon^\top\Delta\Upsilon)^{-1/2}O,
 \qquad O^\top O=I_k.
 \label{eq:loom-component-basis}
\end{equation}
The diagonal matrix $\Upsilon^\top\Delta\Upsilon$ contains the positive component volumes. Within component $a$, each row of this basis is a positive scalar multiple of the $a$th row of $O$. Row normalization removes that scalar. The $k$ resulting rows are distinct because the rows of $O$ are orthonormal. Clustering these rows into $k$ nonempty groups has minimum within-cluster sum of squares zero, attained exactly by the component partition up to permutation.
\end{proof}

\paragraph{Implicit reduced operator}
Using~\eqref{eq:loom-lift} and the Frobenius inner-product identity,
for each original index $i\in\mathcal I_+$, the corresponding row of $F_+v$ equals $p_i^\top\operatorname{smat}(v)p_i$.
Multiplying by $F_+^\top\Delta^{-1}$ yields~\eqref{eq:loom-gram-action}. This proves the implicit operator identity.

\section{Implementation Conventions}
\label{app:loom-implementation}

The implementation specifies the network architecture, preprocessing, initialization, random seeds, Adam settings, and the iteration budgets in Algorithm~\ref{alg:loom-main}. Each training update follows the stated order. When gradients are accumulated over sample chunks, their contributions in~\eqref{eq:loom-batches} are summed before the parameter update. Joint training starts with zero Adam moments.

\paragraph{Sample-based initialization}
For $k$-means++ seeding of $L$, select the first index uniformly from $\{1,\ldots,n\}$. Let $\mathcal S$ be the indices already selected. For each remaining index, compute its squared distance to the closest selected column,
\begin{equation}
 d_i^{\mathrm{seed}}=\min_{j\in\mathcal S}\|z_i-z_j\|_2^2,
 \qquad
 \Pr(i\mid\mathcal S)
 =\frac{d_i^{\mathrm{seed}}}{\sum_{j\notin\mathcal S}d_j^{\mathrm{seed}}}.
 \label{eq:loom-seeding}
\end{equation}
Use this distribution until $m$ distinct indices have been selected. When its denominator is zero, sample uniformly from the remaining indices. Put the selected columns into $L$ in selection order. This is the squared-distance seeding rule~\cite{loom_kmeanspp}. 

\paragraph{Adam update}
For completeness, let $w_t$ be the vector containing the current network parameters and let $g_t$ be its gradient of the specified objective. These are vectorizations used only to define the optimizer. Let $a_t$ and $b_t$ be its first- and second-moment vectors. With $a_0=b_0=0$, coefficient inputs $0\le\rho_1,\rho_2<1$, step size $\eta_t>0$, and $\epsilon_{\mathrm{Adam}}>0$, one update is
\begin{equation}
 \begin{aligned}
 a_t&=\rho_1a_{t-1}+(1-\rho_1)g_t,\\
 b_t&=\rho_2b_{t-1}+(1-\rho_2)(g_t\odot g_t),\\
 \widehat a_t&=a_t/(1-\rho_1^t),\qquad
 \widehat b_t=b_t/(1-\rho_2^t),\\
 w_{t+1}&=w_t-\eta_t\,
 \widehat a_t\oslash(\sqrt{\widehat b_t}+\epsilon_{\mathrm{Adam}}).
 \end{aligned}
 \label{eq:loom-adam}
\end{equation}
Products $\odot$, divisions $\oslash$, and the square root are elementwise. These equations specify Adam~\cite{loom_adam}, without an added weight-decay penalty. For pretraining, $g_t$ differentiates~\eqref{eq:loom-pretrain}; for joint training, it differentiates~\eqref{eq:loom-joint} with fixed factors. The moment vectors have the same dimension as $w_t$.

\paragraph{Factor accuracy}
The thin decomposition in~\eqref{eq:loom-p-update} uses $m$ orthonormal columns, completing the singular bases when zero singular values occur. After each factor update, we check the residual $\|P^\top P-I_m\|_F$ at the chosen working precision.

\paragraph{Spectral accuracy and reproducibility}
The eigensolver retains the analytical mode in~\eqref{eq:loom-leading-vector} and computes the remaining dominant modes on its orthogonal complement. For prescribed residual, orthogonality, and positive-eigenvalue tolerances
$\varepsilon_{\mathrm{eig}}>0$, $\epsilon_{\mathrm{orth}}>0$, and $\varepsilon_{\mathrm{rank}}>0$, respectively, we accept computed eigenpairs satisfying
\begin{equation}
 \begin{aligned}
 \max_{1\leq j\leq k}\|Gv_j-\nu_jv_j\|_2
   &\leq\varepsilon_{\mathrm{eig}},\\
 \max_{1\leq a,b\leq k}
 |v_a^\top v_b-\mathbf1\{a=b\}|
   &\leq\epsilon_{\mathrm{orth}},\\
 \nu_k&>\varepsilon_{\mathrm{rank}}.
 \end{aligned}
 \label{eq:loom-spectral-certificate}
\end{equation}
The iteration budget or working precision is increased as needed to meet these criteria. Eigenvector ordering and random seeds are recorded for reproducibility. When an eigenvalue is repeated at the truncation boundary,
the solver fixes a reproducible orthonormal basis of the selected modes. The active set contains every strictly positive degree in \eqref{eq:loom-degree}. Rows with zero degree are assigned by \eqref{eq:loom-isolate-assignment}.

\paragraph{Final clustering}
Let $\iota(i)$ denote the row of $F_+$ corresponding to the original sample index $i\in\mathcal I_+$. The same ordering is used in $H$ and $\widehat H$. With spectral centers $\mu_a\in\mathbb R^k$, the clustering objective and Lloyd updates are
\begin{equation}
 \begin{aligned}
 \min_{\{\gamma_i\},\{\mu_a\}}\quad
 &\sum_{i\in\mathcal I_+}
 \|\widehat H_{\iota(i):}^{\top}-\mu_{\gamma_i}\|_2^2,\\
 \gamma_i&\leftarrow\arg\min_{1\le a\le k}
 \|\widehat H_{\iota(i):}^{\top}-\mu_a\|_2^2,\\
 \mu_a&\leftarrow\frac1{|\mathcal J_a|}
 \sum_{i\in\mathcal J_a}\widehat H_{\iota(i):}^{\top}.
 \end{aligned}
 \label{eq:loom-kmeans}
\end{equation}
Use $R_{\mathrm{km}}\ge1$ recorded $k$-means++ restarts on these rows. Apply the seeding rule~\eqref{eq:loom-seeding} to the spectral rows, selecting $k$ seeds among the $n_+$ active samples. Each restart takes at most $T_{\mathrm{km}}$ Lloyd steps. Ties between centers use the smallest center index. Repair an empty cluster before its mean update by moving the sample with the largest current squared residual from a cluster containing at least two samples, with ties resolved by the original sample index. Then recompute affected centers. Stop early when assignments do not change. Keep the restart with minimum within-cluster sum of squares. This procedure returns $k$ nonempty groups for $n_+\ge k$, as ensured by~\eqref{eq:loom-lift-rank}.

\section{Complexity Comparison with Baselines}
\label{app:complexity-comparison}

Table~\ref{tab:complexity} separates representation learning from spectral computation. The bounds count sequential arithmetic operations. The numbers of training passes and optimization iterations are held fixed to isolate sample-size scaling. Eigensolver budgets are retained explicitly.

To compare dense, sparse, and implicit spectral operators on the same basis, define
\begin{equation}
\mathcal C_{\mathrm{sp}}(d_{\mathrm{op}},\omega)
={}T_{\mathrm{eig}}(\omega+d_{\mathrm{op}}s_{\mathrm{eig}})+R_{\mathrm{eig}}(d_{\mathrm{op}}s_{\mathrm{eig}}^2+s_{\mathrm{eig}}^3).
\label{eq:comparison-spectral-cost}
\end{equation}
Here, $d_{\mathrm{op}}$ is the operator dimension and $\omega$ bounds the cost of one matrix--vector product. The solver-specific budgets $T_{\mathrm{eig}}$, $s_{\mathrm{eig}}$, and $R_{\mathrm{eig}}$ have the meanings given in Section~\ref{sec:loom-complexity}. For an explicitly stored sparse graph, $\omega$ is its number of nonzero entries. A dense graph instead has $\omega=O(n^2)$. Thus, partial eigensolvers retain the benefit of sparse affinities. The common final $k$-means cost, $O(R_{\mathrm{km}}(T_{\mathrm{km}}+1)nk^2)$, is added once to the sum of the two time columns for the complete pipeline.

\begin{table*}[!t]
\centering
\small
\setlength{\tabcolsep}{4pt}
\renewcommand{\arraystretch}{1.25}
\caption{Time and representation-storage complexity of the compared methods. All entries are inside $O(\cdot)$. The two time columns are additive, with the common final $k$-means cost added once. Fixed optimization budgets are suppressed. Spectral budgets are included through~\eqref{eq:comparison-spectral-cost}. The last column counts the principal coefficient, factor, and affinity matrices, not total working memory. Bounds are obtained from the original algorithms under the implementation conventions stated below.}
\label{tab:complexity}
\begin{tabular}{@{}lc>{\raggedright\arraybackslash}p{2.30in}>{\raggedright\arraybackslash}p{1.98in}>{\raggedright\arraybackslash}p{1.07in}@{}}
\toprule
Method & Deep & Representation / affinity learning & Spectral computation & Matrix storage \\
\midrule
LMVSC~\cite{paper15} & No
& $nd_xm+nm^3$
& $nm^2+m^3$
& $nm$ \\
SGL~\cite{paper36} & No
& $nd_xm+nm^3$
& $nm^2+m^3$\textsuperscript{a}
& $nm$ \\
S$^5$C~\cite{loom_s5c} & No
& $n(d_xm+c_{\ell_1})$
& $\mathcal C_{\mathrm{sp}}(n,nm)$
& $nm$ \\
EnSC~\cite{paper19} & No
& $n^2d_x+n c_{\mathrm{EN}}$
& $\mathcal C_{\mathrm{sp}}(n,ns)$
& $ns$ \\
SSC-OMP~\cite{paper17} & No
& $n^2d_xs$
& $\mathcal C_{\mathrm{sp}}(n,ns)$
& $ns$ \\
A-DSSC~\cite{paper22} & No
& $\mathcal C_C+\mathcal C_{\mathrm{chk}}+e_{\mathrm{act}}$
& $\mathcal C_{\mathrm{sp}}(n,e_{\mathrm{act}})$
& $\mathcal M_C+e_{\mathrm{act}}$ \\
SSSC~\cite{paper18} & No
& $R_{\mathrm{drop}}\rho n^2d_xs$
& $\mathcal C_{\mathrm{sp}}(n,e_{\mathrm{drop}})$
& $e_{\mathrm{drop}}$ \\
LSR~\cite{paper9} & No
& $n^2d_x+\min\{n^3,nd_x^2+d_x^3\}$\textsuperscript{b}
& $\mathcal C_{\mathrm{sp}}(n,n^2)$
& $n^2$ \\
\midrule
PRO-DSC~\cite{loom_pro_dsc} & Yes
& $\begin{aligned}[t]
&(n/n_b)c_{\mathrm{PRO}}(n_b)\\[-2pt]
&\quad+nc_{\mathrm{net}}+n^2d_z
\end{aligned}$
& $\mathcal C_{\mathrm{sp}}(n,n^2)$
& $n^2$ \\
LoomSC & Yes
& $nc_{\mathrm{net}}+nd_zm+nm^2+m^3$
& $\begin{aligned}[t]
&nm+\mathcal C_{\mathrm{sp}}(q,nq)\\[-2pt]
&\quad+nqk
\end{aligned}$
& $nm+d_zm$ \\
\bottomrule
\end{tabular}
\par\vspace{3pt}
\begin{minipage}{\textwidth}
\footnotesize
For LoomSC, $q=m(m+1)/2$. The table reports spectral computation using the operator in~\eqref{eq:loom-gram-action}, without assembling $G$. $s$ bounds the nonzeros per sample for EnSC and SSC-OMP, and per dropout subproblem for SSSC. For SSSC, $\rho$ is the expected retained dictionary fraction and $R_{\mathrm{drop}}$ is the number of dropout dictionaries, giving $e_{\mathrm{drop}}=\min\{n^2,nsR_{\mathrm{drop}}\}$. The SSSC row gives the expected sequential computational cost at a fixed consensus budget. $c_{\ell_1}$ and $c_{\mathrm{EN}}$ are the costs of one restricted Lasso and elastic-net solve, respectively. The former uses at most $m$ representatives; the latter depends on the active-set size. For A-DSSC, $\mathcal C_C$ and $\mathcal M_C$ denote the chosen coefficient solver's computational cost and matrix storage. $\mathcal C_{\mathrm{chk}}$ includes support initialization and global checks, and $e_{\mathrm{act}}$ bounds the entries processed by its support-restricted transport solver. $n_b$ is PRO-DSC's training batch size, and $c_{\mathrm{PRO}}$ is defined in~\eqref{eq:comparison-pro-batch}. Anchor-method spectral costs use a direct thin SVD of the $n\times m$ factor. \textsuperscript{a}SGL updates this representation within its graph-learning loop. \textsuperscript{b}LSR uses the cheaper of the sample-space and feature-space linear solves. Both routes explicitly form the coefficient matrix.
\end{minipage}
\end{table*}

PRO-DSC~\cite{loom_pro_dsc} trains on mini-batches and constructs its final affinity over the complete clustering set. Combining its regularizer costs with neural computation and matrix products gives the following per-batch computational-cost bound:
\begin{equation}
\begin{aligned}
c_{\mathrm{PRO}}(n_b)
={}&n_b c_{\mathrm{net}}+n_b^2d_z+\min\{n_b^3,d_z^3\}\\
&+\mathcal C_{\mathrm{sp}}(n_b,n_b^2).
\end{aligned}
\label{eq:comparison-pro-batch}
\end{equation}
For $1\le n_b\le n$, the $n/n_b$ multiplier counts batches in a full training pass. A fixed number of Sinkhorn iterations is absorbed into the matrix-product term. The separate $n^2d_z$ term accounts for constructing the full-set affinity at inference. The per-sample network cost $c_{\mathrm{net}}$ is evaluated for each method's own architecture.

The regression-based methods differ in how they reduce optimization cost. S$^5$C~\cite{loom_s5c} uses a selection budget of $m$ and solves the final restricted regression for every sample. Its bound includes both selection and regression costs. For EnSC~\cite{paper19}, the active-set solver is left explicit because the original algorithm permits different restricted solvers. A-DSSC~\cite{paper22} likewise permits different coefficient estimators. Each support-restricted transport evaluation is linear in $e_{\mathrm{act}}$. A complete support check additionally costs $O(n^2)$ with stored coefficients, or $O(d_xn^2)$ when feature-space LSR entries are evaluated on demand. For LSR~\cite{paper22} with $d_x<n$, a feature-space solve reduces coefficient construction to $O(nd_x^2+d_x^3+n^2d_x)$. Explicit assembly and storage of the full coefficient matrix remain quadratic in $n$.

The matrix-storage column isolates the representation responsible for affinity scaling. Total working memory also includes the input data, spectral embedding, eigensolver workspace, and regression-solver workspace. Deep methods additionally retain latent features and network state. For LoomSC, these terms are specified in~\eqref{eq:loom-base-memory}, with the corresponding eigensolver workspace. Quadratic feature rows can be generated as needed, so an $n\times q$ feature matrix need not be stored.

The comparison identifies the source of LoomSC's scalability advantage. LMVSC~\cite{paper15}, SGL~\cite{paper36}, and S$^5$C~\cite{loom_s5c} already have linear sample-size dependence for fixed representative and solver budgets. LoomSC retains this dependence while jointly learning nonlinear features and the two self-expression factors. EnSC~\cite{paper19} and SSC-OMP~\cite{paper17} can reduce coefficient storage through sparsity, but their dictionary scans still compare each target with the full dataset. SSSC reduces the dictionary size by dropout, giving the explicit $R_{\mathrm{drop}}\rho n^2$ dependence in the table~\cite{paper18}. PRO-DSC~\cite{loom_pro_dsc} instead improves the training stage through mini-batches, while its full-affinity spectral inference retains quadratic storage. LoomSC avoids explicitly constructing either $n\times n$ matrix. Its $q$-dimensional reduction is exact for the defined squared affinity. For fixed $m$ and solver budgets, both its learned factors and its reduced spectral computation scale linearly with $n$. Thus, LoomSC combines deep representation learning with scalability through the final clustering stage.

\end{document}